\RequirePackage{fix-cm}
\PassOptionsToPackage{dvipsnames, svgnames, x11names, table}{xcolor}
\documentclass{article}
\usepackage{iclr2027_conference,times}

\iclrpreprintcopy   

\usepackage{paper_config}
\usepackage{math_commands}
\usepackage{tikz}
\usepackage{amsmath,amssymb,bm,mathtools}
\usepackage{xcolor}
\usetikzlibrary{calc,positioning,arrows.meta,decorations.pathmorphing,
                decorations.markings,decorations.pathreplacing,shapes.geometric,
                shapes.symbols,shapes.misc,fit,backgrounds,matrix,patterns,
                shadings,fadings,angles,quotes,3d}

\definecolor{egnnBlue}{HTML}{1F6FB2}   
\definecolor{snnRed}{HTML}{C0392B}     
\definecolor{esnnGreen}{HTML}{2E8B57}  
\definecolor{scalarSlate}{HTML}{5B6B7B}
\definecolor{vectorTeal}{HTML}{0E8C8C} 
\definecolor{coordPurple}{HTML}{6C3FA0}
\definecolor{radialOrange}{HTML}{E07B22}
\definecolor{tanBlue}{HTML}{2E78C7}     
\definecolor{inkGray}{HTML}{2B2B2B}
\definecolor{softGray}{HTML}{8A8A8A}
\definecolor{panelGray}{HTML}{F4F5F7}
\definecolor{goodGreen}{HTML}{2E8B57}
\definecolor{badRed}{HTML}{C0392B}

\tikzset{
  >=Stealth,
  font=\small,
  panel/.style={rounded corners=3pt, draw=softGray, line width=0.6pt,
                fill=panelGray, inner sep=8pt},
  titlebar/.style={font=\bfseries\small, inner sep=2pt},
  block/.style={rounded corners=2pt, draw=inkGray, line width=0.7pt,
                fill=white, inner sep=5pt, align=center, minimum height=8mm},
  opbox/.style={block, fill=panelGray},
  scalarbox/.style={block, draw=scalarSlate, fill=scalarSlate!10,
                    text=scalarSlate},
  vectorbox/.style={block, draw=vectorTeal, fill=vectorTeal!10,
                    text=vectorTeal},
  coordbox/.style={block, draw=coordPurple, fill=coordPurple!10,
                   text=coordPurple},
  gnode/.style={circle, draw=inkGray, line width=0.7pt, fill=white,
                minimum size=6mm, inner sep=0pt, font=\footnotesize},
  gnodeA/.style={gnode, fill=egnnBlue!18, draw=egnnBlue},
  gnodeB/.style={gnode, fill=esnnGreen!18, draw=esnnGreen},
  gedge/.style={line width=1.1pt, draw=softGray},
  flow/.style={-{Stealth[length=2.6mm,width=2.2mm]}, line width=0.9pt,
               draw=inkGray},               
  flowfix/.style={flow, dashed, draw=softGray}, 
  transport/.style={-{Stealth[length=2.6mm]}, line width=1.0pt,
                    draw=vectorTeal},        
  vglyph/.style={-{Stealth[length=1.8mm]}, line width=1.0pt},
  ok/.style={text=goodGreen, font=\bfseries},
  bad/.style={text=badRed, font=\bfseries},
}

\providecommand{\cmark}{{\color{goodGreen}\ding{51}}}
\providecommand{\xmark}{{\color{badRed}\ding{55}}}

\providecommand{\ballnode}[4]{%
  \shade[ball color=#3] #1 circle (#2);
  \draw[#3!55!black, line width=0.4pt] #1 circle (#2);
  \node[font=\footnotesize\bfseries, text=white] at #1 {#4};}

\providecommand{\glyphCompress}[2]{%
  \begin{scope}[shift={#1}]
    \draw[#2,line width=0.8pt,fill=#2!14] (0,0) circle (0.30);
    \foreach \a in {0,45,...,315}{\draw[#2,-{Stealth[length=1.3mm]},line width=0.7pt]
       (\a:0.52)--(\a:0.36);}
  \end{scope}}
\providecommand{\glyphShear}[2]{%
  \begin{scope}[shift={#1}]
    \draw[softGray,dashed,line width=0.4pt] (-0.34,-0.30) rectangle (0.34,0.30);
    \draw[#2,line width=0.9pt,fill=#2!14]
       (-0.58,-0.30)--(0.10,-0.30)--(0.58,0.30)--(-0.10,0.30)--cycle;
    \draw[#2,-{Stealth[length=1.3mm]},line width=0.7pt] (-0.05,0.42)--(0.42,0.42);
    \draw[#2,-{Stealth[length=1.3mm]},line width=0.7pt] (0.05,-0.42)--(-0.42,-0.42);
  \end{scope}}
\providecommand{\glyphTorsion}[2]{%
  \begin{scope}[shift={#1}]
    \draw[softGray,dashed,line width=0.4pt] (0,0) circle (0.34);
    \draw[#2,-{Stealth[length=1.7mm]},line width=0.95pt] (40:0.40) arc (40:300:0.40);
    \fill[#2] (0,0) circle (0.028);
  \end{scope}}

\title{Equivariant Sheaf Neural Networks: Learning Geometric Transport on Graphs}

\author{
  \textbf{Alessio Borgi}\textsuperscript{\rm 1,3}\thanks{Equal contribution. \textsuperscript{\textdagger}Corresponding author: \texttt{alessio.borgi@uniroma1.it, ab3352@cam.ac.uk}}\quad
  \textbf{Mario Severino}\textsuperscript{\rm 1,2}\footnotemark[1]\quad
  \textbf{Fabrizio Silvestri}\textsuperscript{\rm 3}\quad
  \textbf{Pietro Li\`o}\textsuperscript{\rm 1}
  \\
  \small\textsuperscript{\rm 1}Department of Computer Science and Technology, University of Cambridge \\
  \small\textsuperscript{\rm 2}Department of Information Engineering, University of Padua \\
  \small\textsuperscript{\rm 3}Department of Computer, Control and Management Engineering, Sapienza University of Rome
}

\begin{document}
\maketitle

\begin{abstract}
Equivariant graph neural networks constrain how geometric features transform,
but efficient first-order architectures remain limited in how vector information
can be transported across graph edges. We introduce \textsc{ESNN}, an
Equivariant Sheaf Neural Network that learns directed, matrix-valued transport
between neighboring vector features while preserving exact Euclidean
equivariance. Rather than increasing representation order, ESNN retains
first-order scalar and vector features and places additional geometric
expressivity in the edge interactions themselves. Its transport mechanisms
range from isotropic transformations to feature-conditioned rotations and
anisotropic radial--tangential actions, allowing vector messages to be
transformed according to the local geometry before aggregation. We also equip
ESNN with a controlled symmetry-relaxation mechanism for systems with a
preferred ambient direction, preserving exact equivariance to the remaining
symmetry while recovering full $E(n)$-equivariance as a special case. Across
diverse geometric learning problems spanning physical dynamics, point clouds,
and molecular systems, ESNN improves predictive performance over first-order
equivariant baselines, learns a latent preferred direction from data, and
remains robust to unseen rotations. These results establish edge-wise geometric
transport as a source of equivariant expressivity that does not require
higher-order representations.
\end{abstract}

\vspace{-0.3cm}
\section{Introduction}
\label{sec:intro}
Many physical systems, from molecular dynamics to fluid mechanics, are
naturally represented as geometric graphs embedded in Euclidean space
~\citep{satorras2022enequivariantgraphneural,
brandstetter2021geometric,du2023new,wang2022visnet,
aykent2025gotennet}. Their states may contain invariant scalar attributes
together with vectors or tensors that transform with the geometry
~\citep{schutt2021equivariant,simeon2306tensornet,
liao2024equiformerv2}. Equivariant neural networks encode these transformation
laws directly, ensuring predictable behavior under translations, rotations,
and, when appropriate, reflections. Equivariance alone, however, does not
determine how geometric information should be transformed as it passes
between neighboring nodes. Many interactions are intrinsically directional:
longitudinal and transverse responses, shear, and anisotropic propagation
depend on local orientation rather than on distance alone. Existing geometric architectures capture this structure in different ways.
Cartesian models work efficiently with invariant scalars and first-order
vectors~\citep{satorras2022enequivariantgraphneural,schutt2021equivariant},
while steerable models enrich the propagated representations through
higher-order features, spherical harmonics, and tensor products
~\citep{thomas2018tensor,batatia2022design,batatia2022mace}. We instead retain a first-order scalar--vector representation and increase expressivity through richer geometric transport of vector features. This perspective is closely related to Sheaf Neural Networks (SNNs), where interactions between local feature spaces are mediated by matrix-valued maps
~\citep{hansen2020sheaf,bodnar2022neural}. Generic sheaf maps, however, are not inherently constrained to respect the transformation laws of geometric vector features. ESNN addresses this by constraining matrix-valued edge transport to satisfy Euclidean equivariance. Each transport separates an $O(n)$-covariant spatial action from invariant mixing across vector channels, allowing the geometric transformation applied to a message to depend on the local interaction while retaining scalar--vector representations throughout the network. We instantiate this framework through a family of transport mechanisms, from
identity and isotropic cases to feature-conditioned rotations and anisotropic
radial--tangential transport. ESNN also accommodates systems with a preferred ambient direction, such as those governed by gravity or imposed flow. This direction may be specified from prior knowledge or learned from data. The resulting model remains exactly equivariant to the subgroup that preserves the preferred direction, while full $E(n)$-equivariance is recovered when the directional conditioning is inactive.
Figure~\ref{fig:esnn-teaser} contrasts EGNN-style isotropic message passing,
unconstrained sheaf transport, and ESNN's equivariant matrix-valued transport.

\begin{figure}[t]
\centering
\resizebox{\textwidth}{!}{%
\input{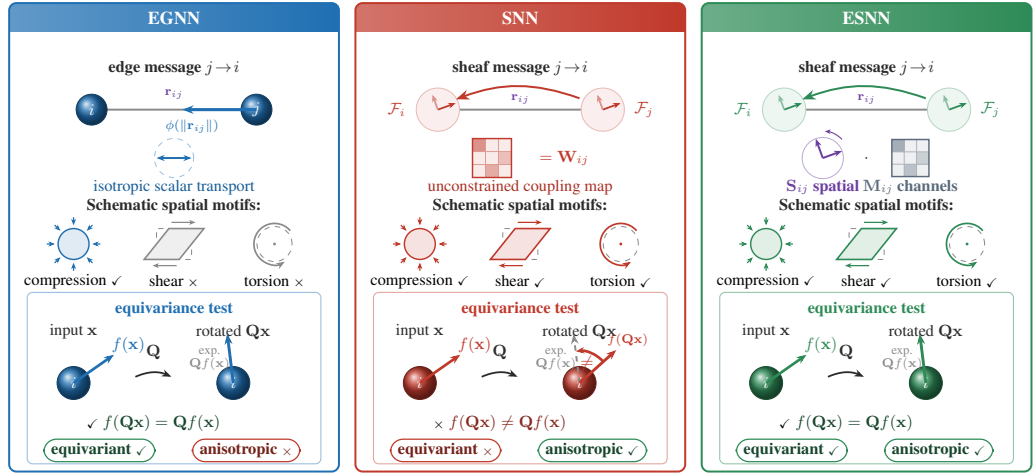}%
}
\caption{
\emph{Edge-wise geometric transport in EGNN-style models, generic sheaf
networks, and ESNNs.} EGNN-style models couple invariant weights with relative
displacements, while sheaf networks allow matrix-valued maps between local
feature spaces. ESNN combines matrix-valued transport with exact equivariance
through $O(n)$-covariant spatial actions and invariant channel mixing.
}
\label{fig:esnn-teaser}
\vspace{-0.6cm}
\end{figure}

Our main contributions are summarized as follows:
\textbf{(i)} We introduce \emph{Equivariant Sheaf Neural Networks}, a
first-order Cartesian framework that learns matrix-valued transport of vector
features across graph edges while preserving $E(n)$-equivariance.
\textbf{(ii)} We develop a family of equivariant transport mechanisms that
progressively enrich how vector features interact across edges, from identity
and isotropic transport to feature-conditioned orthogonal transformations and
anisotropic radial--tangential transport.
\textbf{(iii)} We extend ESNN to systems with a preferred ambient direction,
while retaining exact equivariance to the corresponding stabilizer subgroup,
and evaluate the framework across particle dynamics, mesh-based physical
simulation, point-cloud classification, and molecular property prediction.

\vspace{-0.3cm}
\section{Related Work and Background}
\label{sec:back}
\vspace{-0.1cm}

\noindent\textbf{Related Work.}
ESNN combines ideas from equivariant message passing and sheaf-based graph
learning. Cartesian architectures such as
EGNN~\citep{satorras2022enequivariantgraphneural} and scalar--vector models
such as PaiNN and GVP~\citep{schutt2021equivariant,jing2021equivariant}
achieve efficient $E(n)$-equivariance using invariant scalars together with
low-order covariant features. Steerable architectures instead propagate higher-order representations and couple them through equivariant tensor products
~\citep{thomas2018tensor,fuchs2020se,brandstetter2021geometric,
batatia2022mace,liao2024equiformerv2}. ESNN remains in the first-order
scalar--vector setting, but increases geometric expressivity by enriching the
transport applied to vector features rather than by introducing higher-order
representations. The transport viewpoint is closely related to SNNs, where scalar edge weights are replaced by linear maps between local feature spaces
~\citep{hansen2020sheaf,bodnar2022neural}. Subsequent work has explored more
structured forms of transport, including orthogonal connection maps for
geometric propagation~\citep{barbero2022sheaf} and directed and asymmetric sheaf propagation ~\citep{ribeiro2025cooperative,fiorini2025sheaves}. Closely related
copresheaf formulations place similar emphasis on directed maps and their
composition across the graph~\citep{hajij2025copresheaf}. ESNN adopts this
matrix-valued view of message transport, but constrains the learned maps to
respect the Euclidean transformation laws of geometric vector features. A complementary line of work studies settings in which external structure
reduces the symmetry of the problem. Subequivariant models enforce the
stabilizer of a prescribed direction~\citep{han2022learning}, while
relaxed-equivariant approaches learn controlled departures from an underlying
symmetry~\citep{wang2022approximately,hofgard2024relaxed}. ESNN incorporates
a preferred direction within the transport framework, retaining exact
stabilizer equivariance and recovering full $E(n)$-equivariance when the
directional conditioning is inactive. Further discussion of related work is provided in Appendix~\ref{app:related}.

\subsection{Geometric Background}
\label{sec:background}

We consider a physical system represented by a graph
$\mathcal{G}=(\mathcal{V},\mathcal{E})$, where each node
$i\in\mathcal{V}$ has coordinates
$\mathbf{x}_i\in\mathbb{R}^n$. We seek message-passing operators that respect
Euclidean symmetries while allowing interactions to depend on local geometry.
We first specify the relevant feature transformations and then introduce the
cellular-sheaf viewpoint that motivates the transport construction of
Section~\ref{sec:canonical_form}.

\noindent\textbf{Equivariance.}
The Euclidean group $E(n)=O(n)\ltimes\mathbb{R}^n$ acts on coordinates and
relative displacements as:
\begin{equation}
    \mathbf{x}_i
    \mapsto
    Q\mathbf{x}_i+\mathbf{t},
    \qquad
    \mathbf{r}_{ij}
    =
    \mathbf{x}_i-\mathbf{x}_j
    \mapsto
    Q\mathbf{r}_{ij},
    \qquad
    Q\in O(n),
    \qquad
    \mathbf{t}\in\mathbb{R}^n
\end{equation}
Relative displacements are therefore translation invariant, and operations
built from relative geometry only need to preserve the corresponding
$O(n)$ transformation laws.

\noindent\textbf{Feature Representation.}
Each node carries invariant scalar features
$\mathbf{s}_i\in\mathbb{R}^{c_s}$ and covariant vector features
$\mathbf{V}_i\in\mathbb{R}^{n\times c_v}$, collected as:
\begin{equation}
    \mathbf{h}_i
    =
    (\mathbf{s}_i,\mathbf{V}_i),
    \qquad
    \mathbf{s}_i
    \mapsto
    \mathbf{s}_i,
    \qquad
    \mathbf{V}_i
    \mapsto
    Q\mathbf{V}_i
\end{equation}
The corresponding local feature space is
    $\mathcal{F}(i)
    =
    \mathbb{R}^{c_s}
    \oplus
    \left(
        \mathbb{R}^{n}\otimes\mathbb{R}^{c_v}
    \right)$,
which we interpret below as the node \emph{stalk}. ESNN transports the vector
component geometrically, while scalar features are propagated through
invariant operations.

\noindent\textbf{Cellular Sheaves and Transport.}
Standard message passing implicitly assumes that neighboring features can be
compared and aggregated in a common space. Cellular sheaves generalize this
picture by assigning local vector spaces to nodes and edges together with
linear restriction maps~\citep{hansen2020sheaf,bodnar2022neural}. For an edge
$e=\{i,j\}$, a cellular sheaf provides an edge stalk $\mathcal{F}(e)$ and maps
$\rho_{i\to e}:\mathcal{F}(i)\rightarrow\mathcal{F}(e)$ and
$\rho_{j\to e}:\mathcal{F}(j)\rightarrow\mathcal{F}(e)$, which express the two
node features in a common edge space. These maps define the cellular
coboundary and the associated sheaf Laplacian
$\mathbf{L}_{\mathcal{F}}
=
\delta_{\mathcal{F}}^{*}\delta_{\mathcal{F}}$.
For an edge $e=\{i,j\}$, its off-diagonal term induces, up to the conventional
minus sign, the node-to-node map:
\begin{equation}
    \rho_{i\to e}^{*}\rho_{j\to e}
    :
    \mathcal{F}(j)\rightarrow\mathcal{F}(i)
\end{equation}

This induced coupling provides the starting point for the ESNN construction.
Rather than introducing separate incidence maps through an intermediate edge
stalk, ESNN directly parameterizes transport between neighboring node
features. For geometric vector channels, these maps must additionally respect
the ambient $O(n)$ action, which motivates the spatial--channel formulation
developed in Section~\ref{sec:canonical_form}. The connection to classical
sheaf diffusion and directed transport is developed further in Appendix~\ref{app:sheaf_transport}.

\section{Equivariant Spatial Transport}
\label{sec:canonical_form}
The vector features introduced in Section~\ref{sec:background} naturally
decompose into two factors: a spatial component $\mathbb{R}^n$, which carries
the $O(n)$ action, and a channel component $\mathbb{R}^{c_v}$, which remains
invariant under the group. ESNN exploits this structure by
factorizing each edge transport into a spatial action on geometric vectors
and a channel action across vector features. Let $\mathcal{C}_{ij}$ denote the local context for the directed interaction
$j\rightarrow i$, including the relative displacement
$\mathbf{r}_{ij}$ and, when used, covariant features at the incident nodes.
We write $Q\!\cdot\!\mathcal{C}_{ij}$ for the simultaneous action of
$Q\in O(n)$ on all covariant quantities in this context.

\begin{definition}[\emph{Equivariant Transport Map}]
\label{def:canonical_transport}
For a directed interaction $j\rightarrow i$ and a fixed local edge context,
ESNN defines a transport map $\mathcal{T}_{i\leftarrow j}:
    \mathbb{R}^{n\times c_v}
    \longrightarrow
    \mathbb{R}^{n\times c_v}$ of the form:
\begin{equation}
    \mathcal{T}_{i\leftarrow j}(\mathbf{V}_j)
    =
    \sum_{k=1}^{K}
    \mathbf{S}_{ij}^{(k)}
    \mathbf{V}_j
    \mathbf{M}_{ij}^{(k)}
    \label{eq:canonical_transport}
\end{equation}
where $K$ denotes the number of components in the transport expansion. Each
term combines two actions:
\begin{itemize}
    \item The \emph{spatial operator}
    $\mathbf{S}_{ij}^{(k)}\in\mathbb{R}^{n\times n}$ acts on the spatial
    axis and transforms covariantly:
    \begin{equation}
        \mathbf{S}_{ij}^{(k)}(Q\!\cdot\!\mathcal{C}_{ij})
        =
        Q\,
        \mathbf{S}_{ij}^{(k)}(\mathcal{C}_{ij})
        Q^\top \qquad \forall \space\space Q\in O(n)
        \label{eq:covariance_condition}
    \end{equation}

    \item The \emph{channel operator}
    $\mathbf{M}_{ij}^{(k)}\in\mathbb{R}^{c_v\times c_v}$ mixes vector
    channels and is constructed from $O(n)$-invariant edge features:
    \begin{equation}
        \mathbf{M}_{ij}^{(k)}
        (Q\!\cdot\!\mathcal{C}_{ij}) =
        \mathbf{M}_{ij}^{(k)}
        (\mathcal{C}_{ij}) \qquad \forall \space\space Q\in O(n)
        \label{eq:channel_invariance_condition}
    \end{equation}
\end{itemize}
\end{definition}

These two transformation laws are sufficient to make the resulting transport
equivariant.

\begin{proposition}[Equivariance of the Transport Map]
\label{prop:transport_equivariance}
Suppose that, for every component $k$, the spatial operator satisfies:
\begin{equation}
    \mathbf{S}_{ij}^{(k)}(Q\!\cdot\!\mathcal{C}_{ij})
    =
    Q\mathbf{S}_{ij}^{(k)}(\mathcal{C}_{ij})Q^\top
\end{equation}
and the channel operator $\mathbf{M}_{ij}^{(k)}$ is unchanged under the
$O(n)$ action. Then the transport map in
Equation~\ref{eq:canonical_transport} is $O(n)$-equivariant:
\begin{equation}
    \mathcal{T}'_{i\leftarrow j}(Q\mathbf{V}_j)
    =
    Q\,\mathcal{T}_{i\leftarrow j}(\mathbf{V}_j)
    \label{eq:transport_equivariance}
\end{equation}
where $\mathcal{T}'_{i\leftarrow j}$ denotes the transport evaluated from the
transformed geometric inputs. (Proof is provided in Appendix~\ref{app:proof_transport_equivariance}).
\end{proposition}

The transport associated with each edge is conditioned on the local geometric
context, which may include both invariant quantities and covariant features.
This allows the operator in Equation~\ref{eq:canonical_transport} to adapt to
the current interaction while preserving the scalar--vector representation of
the network. In ESNN, the invariant channel maps are parameterized as:
{\setlength{\abovedisplayskip}{4pt}
 \setlength{\belowdisplayskip}{4pt}
 \setlength{\abovedisplayshortskip}{2pt}
 \setlength{\belowdisplayshortskip}{2pt}
\begin{equation}
    \mathbf{M}_{ij}^{(k)}
    =
    \mathbf{W}^{(k)}
    \mathbf{D}\!\left(\boldsymbol{\eta}_{ij}^{(k)}\right)
    \label{eq:channel_mixing}
\end{equation}
}
\noindent where $\mathbf{W}^{(k)}$ is shared across edges and
$\boldsymbol{\eta}_{ij}^{(k)}\in\mathbb{R}^{c_v}$ is an edge-dependent
gating vector predicted from invariant edge features, with
$\mathbf{D}(\boldsymbol{\eta}_{ij}^{(k)})$ denoting the corresponding
diagonal matrix. The shared transformation mixes vector channels, while the gate
adapts this mixing to the local interaction. In this way,
Equation~\ref{eq:canonical_transport} separates channel mixing from geometric
transformation: the channel maps act across vector features, while the
spatial operators determine how those features are transformed in the ambient
space. Section~\ref{sec:transports} introduces the resulting spatial
transport families, with further details on the decomposition provided in
Appendix~\ref{app:canonical_derivations}.

\begin{figure}[t]
\centering
\resizebox{\textwidth}{!}{%
\input{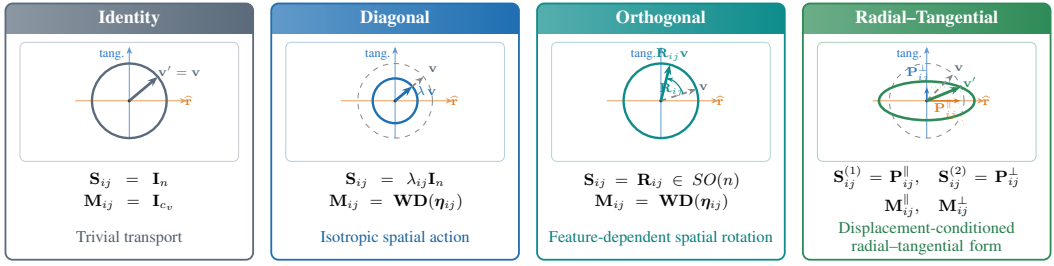}%
}
\caption{\emph{ESNN transport families.}
Identity leaves the vector representation unchanged; Diagonal applies an
edge-dependent isotropic spatial scaling; Orthogonal learns a
feature-conditioned spatial rotation; and Radial--Tangential acts
independently on components parallel and orthogonal to the relative
displacement.}
\label{fig:transports}
\end{figure}

\section{Equivariant Transport Maps}
\label{sec:transports}
We now instantiate the spatial--channel decomposition of
Section~\ref{sec:canonical_form} through four transport families, illustrated
in Figure~\ref{fig:transports}. They differ in the spatial action applied to a
vector message, while their learned coefficients and channel maps are
constructed from $O(n)$-invariant edge features.

\noindent\textbf{Identity Transport.}
\label{sec:transport_0}
The simplest choice leaves the vector message unchanged:
\begin{equation}
    \mathcal{T}_{i\leftarrow j}(\mathbf{V}_j)
    =
    \mathbf{V}_j
    \label{eq:identity_transport}
\end{equation}
corresponding to
$\mathbf{S}_{ij}=\mathbf{I}_n$ and
$\mathbf{M}_{ij}=\mathbf{I}_{c_v}$.
Identity Transport therefore provides the reference case in which vector
features are exchanged without an edge-dependent transformation.

\noindent\textbf{Diagonal Transport.}
\label{sec:transport_0.5}
A first extension allows edge-dependent modulation while remaining isotropic
in physical space. We set
$\mathbf{S}_{ij}=\lambda_{ij}\mathbf{I}_n$, where
$\lambda_{ij}$ is predicted from invariant edge features, giving:
\begin{equation}
    \mathcal{T}_{i\leftarrow j}(\mathbf{V}_j)
    =
    \lambda_{ij}
    (\mathbf{V}_j\mathbf{W})
    \mathbf{D}(\boldsymbol{\eta}_{ij})
    \label{eq:diagonal_transport}
\end{equation}
The same spatial scaling is applied in every direction, while the channel map
remains adaptive to the local interaction.

\noindent\textbf{Orthogonal Transport.}
\label{sec:transport_II}
Diagonal Transport remains spatially isotropic. Orthogonal Transport instead
derives a feature-dependent rotation from the vector states at the two
endpoints of an interaction. We define:
\begin{equation}
    \mathbf{C}_{ij}
    =
    \mathbf{V}_j\mathbf{V}_i^\top,
    \qquad
    \widetilde{\mathbf{C}}_{ij}
    =
    \frac{\mathbf{C}_{ij}}
    {\|\mathbf{C}_{ij}\|_F+\varepsilon},
    \qquad
    \mathbf{\Omega}_{ij}
    =
    \widetilde{\mathbf{C}}_{ij}
    -
    \widetilde{\mathbf{C}}_{ij}^{\top}
    \label{eq:cross_covariance}
\end{equation}
and use an invariant edge coefficient $\beta_{ij}$ to form:
\begin{equation}
    \mathbf{R}_{ij}
    =
    \exp\!\left(\beta_{ij}\mathbf{\Omega}_{ij}\right)
    \in SO(n),
    \qquad
    \mathcal{T}_{i\leftarrow j}(\mathbf{V}_j)
    =
    \mathbf{R}_{ij}
    (\mathbf{V}_j\mathbf{W})
    \mathbf{D}(\boldsymbol{\eta}_{ij})
    \label{eq:orthogonal_transport}
\end{equation}

Since $\mathbf{R}_{ij}$ is inferred from the evolving vector features, the
spatial transformation adapts to the local feature geometry rather than to
relative displacement alone. Under a global orthogonal transformation,
$\mathbf{\Omega}_{ij}$ and $\mathbf{R}_{ij}$ transform by conjugation, so
Proposition~\ref{prop:transport_equivariance} applies; the derivation is given
in Appendix~\ref{app:proof_lemma1}. Here \emph{Orthogonal} refers to the
spatial factor $\mathbf{R}_{ij}$, while the full edge map also includes
learned channel mixing and gating.

\noindent\textbf{Radial--Tangential Transport.}
\label{sec:transport_III}
Many geometric interactions distinguish motion along an edge from motion in
the transverse directions.
For $\mathbf{r}_{ij}\neq\mathbf{0}$, let
$\widehat{\mathbf{r}}_{ij}
=
\mathbf{r}_{ij}/\|\mathbf{r}_{ij}\|$ and define:
\begin{equation}
    \mathbf{P}^{\parallel}_{ij}
    =
    \widehat{\mathbf{r}}_{ij}
    \widehat{\mathbf{r}}_{ij}^{\top},
    \qquad
    \mathbf{P}^{\perp}_{ij}
    =
    \mathbf{I}_n-\mathbf{P}^{\parallel}_{ij}
    \label{eq:radial_tangential_projectors}
\end{equation}
Radial--Tangential Transport applies independent channel transformations to
the radial and tangential components:
\begin{equation}
    \mathbf{M}^{\parallel}_{ij}
    =
    \mathbf{W}_{\parallel}
    \mathbf{D}(\boldsymbol{\eta}^{\parallel}_{ij}),
    \qquad
    \mathbf{M}^{\perp}_{ij}
    =
    \mathbf{W}_{\perp}
    \mathbf{D}(\boldsymbol{\eta}^{\perp}_{ij})
\end{equation}
and transports:
\begin{equation}
    \mathcal{T}_{i\leftarrow j}(\mathbf{V}_j)
    =
    \mathbf{P}^{\parallel}_{ij}
    \mathbf{V}_j
    \mathbf{M}^{\parallel}_{ij}
    +
    \mathbf{P}^{\perp}_{ij}
    \mathbf{V}_j
    \mathbf{M}^{\perp}_{ij}
    \label{eq:transport_III}
\end{equation}
Both projectors transform by conjugation under $O(n)$, so the transport
satisfies Proposition~\ref{prop:transport_equivariance}; the explicit proof is
given in Appendix~\ref{app:proof_lemma2}. The projectors need not be formed
explicitly, and their action scales linearly with the ambient dimension $n$. When relative displacement is the only covariant geometric input, the
radial--tangential form is also complete.
\begin{theorem}[Displacement-Conditioned Equivariant Transport]
\label{thm:maximal_expressivity}
Let $n\geq2$, and consider a linear transport on
$\mathbb{R}^{n}\otimes\mathbb{R}^{c_v}$ whose only covariant geometric
conditioning variable is a non-zero relative displacement
$\mathbf{r}_{ij}$, with arbitrary additional $O(n)$-invariant scalar
conditioning. Every such $O(n)$-equivariant transport can be written as:
\begin{equation}
    \mathcal{T}_{i\leftarrow j}(\mathbf{V}_j)
    =
    \mathbf{P}^{\parallel}_{ij}
    \mathbf{V}_j
    \mathbf{A}_{ij}
    +
    \mathbf{P}^{\perp}_{ij}
    \mathbf{V}_j
    \mathbf{B}_{ij}
\end{equation}
for invariant channel endomorphisms
$\mathbf{A}_{ij},\mathbf{B}_{ij}
\in\mathbb{R}^{c_v\times c_v}$.
The proof is given in Appendix~\ref{app:proof_thm1}.
\end{theorem}

Thus, with relative displacement as the only covariant geometric input,
every admissible spatial action decomposes into radial and tangential
components, each paired with an invariant channel map. ESNN instantiates this
structure through the parameterization in Equation~\ref{eq:channel_mixing}. Additional learned covariant features can support richer spatial operators, as described in Appendix~\ref{app:feature_conditioned_transport}.

\section{The ESNN Architecture}
\label{sec:layer}
An ESNN layer combines the equivariant transports of
Section~\ref{sec:transports} with invariant scalar messaging and, when
required, an equivariant update of the geometry. At node $i$, the hidden
state consists of invariant scalar features
$\mathbf{s}_i\in\mathbb{R}^{c_s}$ and covariant vector features
$\mathbf{V}_i\in\mathbb{R}^{n\times c_v}$.

\noindent\textbf{Input representations.}
Scalar descriptors are embedded in the invariant branch. When geometric
vector quantities are available directly from the input, such as velocities,
they are used to initialize $\mathbf{V}_i$. For tasks without native vector inputs, we construct task-specific covariant lifts $\mathbf{L}_i$ from the input geometry $\mathbf{X}$ satisfying:
\begin{equation}
    \mathbf{L}_i(Q\mathbf{X}+\mathbf{t})
    =
    Q\,\mathbf{L}_i(\mathbf{X})
    \label{eq:vector_lift}
\end{equation}
so that the vector representation obeys the correct transformation law. The lifts used for each task are described in Appendix~\ref{app:experimental_details}.

\noindent\textbf{Edge context and message transport.}
For each directed interaction $j\rightarrow i$, ESNN constructs an invariant
edge representation:
\begin{equation}
    \mathbf{z}_{ij}
    =
    \operatorname{Inv}\!\left(
        \mathbf{s}_i,\mathbf{s}_j,
        \mathbf{V}_i,\mathbf{V}_j,
        \mathbf{r}_{ij},
        \mathbf{e}^{\mathrm{attr}}_{ij}
    \right)
    \label{eq:edge_context}
\end{equation}
from distances, vector norms, and inner products, together with optional
invariant edge attributes. This context determines the gates and coefficients
of the transport maps introduced in Section~\ref{sec:transports}. Vector and
scalar messages are then given by:
\begin{equation}
    \mathbf{m}^{V}_{i\leftarrow j}
    =
    \mathcal{T}_{i\leftarrow j}(\mathbf{V}_j),
    \qquad
    \mathbf{m}^{s}_{i\leftarrow j}
    =
    \mathbf{s}_j+\phi_s(\mathbf{z}^{s}_{ij})
    \label{eq:esnn_messages}
\end{equation}
where $\phi_s$ is a learnable map acting on the invariant edge representation. Geometric information can therefore influence both scalar and vector messages while preserving their respective transformation laws.

\noindent\textbf{Transport aggregation.}
Messages are aggregated using invariant normalized weights $\nu_{ij}$, with
$\nu_{ii}$ denoting the identity self-loop contribution:
\begin{equation}
    \mathbf{V}^{\mathrm{agg}}_i
    =
    \nu_{ii}\mathbf{V}_i
    +
    \sum_{j\in\mathcal{N}(i)}
    \nu_{ij}\,
    \mathbf{m}^{V}_{i\leftarrow j},
    \qquad
    \mathbf{s}^{\mathrm{agg}}_i
    =
    \nu_{ii}\mathbf{s}_i
    +
    \sum_{j\in\mathcal{N}(i)}
    \nu_{ij}\,
    \mathbf{m}^{s}_{i\leftarrow j}
    \label{eq:esnn_aggregation}
\end{equation}
By default, $\nu_{ij}$ are degree-normalized edge weights, though invariant
attention weights can be used in the same way. Because the two directions of an edge may carry different transport maps, the resulting operator is generally directional and need not be self-adjoint. With symmetric aggregation weights, adjoint transport across opposite directions yields a self-adjoint specialization. Its relation to sheaf and quiver operators is discussed further in Appendix~\ref{app:sheaf_transport}.

\noindent\textbf{Feature update.}
The aggregated representations are updated within their respective feature
spaces:
\begin{equation}
    \mathbf{V}'_i
    =
    \mathbf{V}_i
    +
    \sigma_V\!\left(
        \mathbf{V}^{\mathrm{agg}}_i\mathbf{W}_V
    \right),
    \qquad
    \mathbf{s}'_i
    =
    \mathbf{s}_i
    +
    \sigma_s\!\left(
        \mathbf{W}_s\mathbf{s}^{\mathrm{agg}}_i
    \right)
    \label{eq:node_update}
\end{equation}
where $\mathbf{W}_V$ acts only across vector channels,
$\sigma_V(\mathbf{v})=a(\|\mathbf{v}\|)\mathbf{v}$ is applied channel-wise,
and $\sigma_s$ denotes the scalar nonlinearity. The channel mixing and radial
nonlinearity preserve the vector transformation law, while the scalar branch
remains invariant.

\noindent\textbf{Coordinate update.}
For tasks in which the geometry evolves over time, ESNN optionally updates the
node coordinates through an invariantly weighted sum of relative
displacements:
\begin{equation}
    \Delta\mathbf{x}_i
    =
    \frac{1}{|\mathcal{N}(i)|}
    \sum_{j\in\mathcal{N}(i)}
    \gamma_{ij}\mathbf{r}_{ij},
    \qquad
    \mathbf{x}'_i
    =
    \mathbf{x}_i+\Delta\mathbf{x}_i
    \label{eq:coord_update}
\end{equation}
where $\gamma_{ij}$ is an invariant scalar predicted from node and edge
features. This preserves a simple first-order equivariant coordinate update while allowing the latent vector interactions to use the richer transport mechanisms introduced above. Task-specific dynamical updates are described in
Appendix~\ref{app:experimental_details}.

These operations preserve the appropriate transformation laws of the scalar,
vector, and coordinate features, yielding an $E(n)$-equivariant ESNN layer.

\begin{theorem}[$E(n)$-Equivariance]
\label{thm:layer_equivariance}
Assume that the graph topology is fixed or constructed from
$E(n)$-invariant quantities, that edge attributes are $O(n)$-invariant, and
that the transport maps satisfy
Proposition~\ref{prop:transport_equivariance}. Assume also that
$\mathbf{r}_{ij}\neq\mathbf{0}$ whenever a normalized edge direction is
used. In the absence of the symmetry-relaxing inputs introduced in
Section~\ref{sec:soft_symmetry}, the ESNN layer is $E(n)$-equivariant:
\begin{equation}
    (\mathbf{x}_i,\mathbf{V}_i,\mathbf{s}_i)
    \mapsto
    (Q\mathbf{x}_i+\mathbf{t},Q\mathbf{V}_i,\mathbf{s}_i)
    \quad\Longrightarrow\quad
    (\mathbf{x}'_i,\mathbf{V}'_i,\mathbf{s}'_i)
    \mapsto
    (Q\mathbf{x}'_i+\mathbf{t},Q\mathbf{V}'_i,\mathbf{s}'_i)
\end{equation}
The proof is given in Appendix~\ref{app:proof_thm2}.
\end{theorem}

Theorem~\ref{thm:layer_equivariance} shows that richer edge-wise transport
can be introduced without sacrificing the Euclidean equivariance of the
architecture.

\section{Controlled Symmetry Relaxation}
\label{sec:soft_symmetry}

Full $E(n)$-equivariance is appropriate when the ambient space has no
distinguished direction. External structure such as gravity or background flow
can instead reduce the symmetry of the dynamics
\citep{han2022learning,wang2024discovering}. ESNN incorporates this directional structure through an additional scalar conditioning while leaving the underlying spatial transport families unchanged.

Let $\mathbf{g}\in\mathbb{R}^{n}$ denote a global preferred direction,
either specified a priori or learned as a model parameter. At layer $\ell$,
the invariant edge context is augmented with:
\begin{equation}
    z_{ij}^{\mathbf{g},(\ell)}
    =
    \lambda_{\mathbf{g}}^{(\ell)}
    \left\langle
        \mathbf{r}_{ij},\mathbf{g}
    \right\rangle,
    \qquad
    \mathbf{z}_{ij}^{(\ell)}
    \leftarrow
    \left[
        \mathbf{z}_{ij}^{(\ell)},
        z_{ij}^{\mathbf{g},(\ell)}
    \right]
    \label{eq:relaxed_embed}
\end{equation}
where $\lambda_{\mathbf{g}}^{(\ell)}$ is a learnable scalar initialized at
zero. Unlike the invariant features used by the fully equivariant model, this
projection records the orientation of an edge relative to the preferred
direction. It can therefore influence the transport coefficients and channel
gates through the same edge networks used elsewhere in the architecture. Figure~\ref{fig:symmetry-relaxation} illustrates the resulting symmetry
regimes: full $E(n)$-equivariance when the directional pathway is inactive,
and $E_{\mathbf g}(n)$-equivariance when it is active.

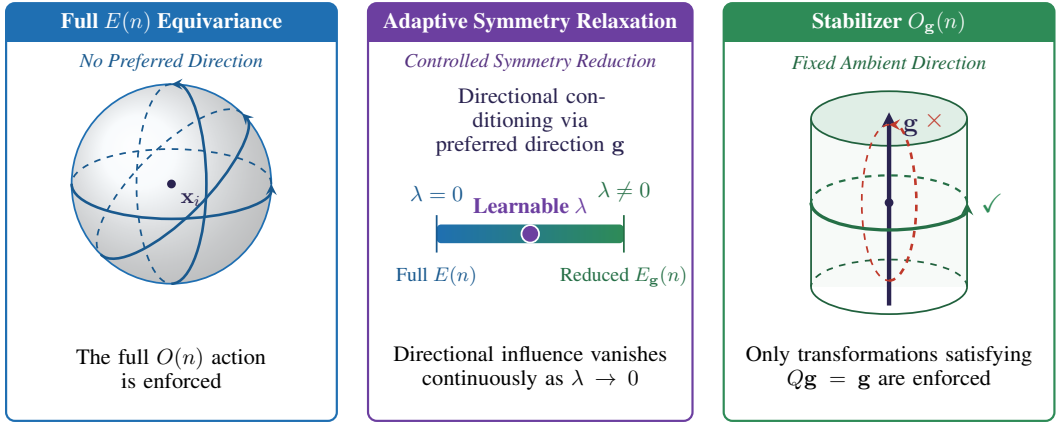
\begin{figure}[!ht]
\centering

\resizebox{\textwidth}{!}{%
%
%
%

\begin{tikzpicture}[font=\normalsize]


\providecolor{fieldDark}{HTML}{2A2350}
\providecolor{egnnBlue}{HTML}{2B60DE}
\providecolor{esnnGreen}{HTML}{00A86B}
\providecolor{coordPurple}{HTML}{6A0DAD}
\providecolor{badRed}{HTML}{D32F2F}
\providecolor{inkGray}{HTML}{666666}

\def\PW{5.3}      
\def\PH{6.7}      
\def\CW{5.2}      
\def\GAP{0.5}     

\coordinate (L) at (0,0);
\coordinate (M) at ($(L)+(\PW+\GAP,0)$);
\coordinate (R) at ($(M)+(\CW+\GAP,0)$);

\pgfmathsetmacro{\SPAN}{2*\PW+\CW+2*\GAP}
\pgfmathsetmacro{\MID}{\SPAN/2}


\providecommand{\symrelaxpanelframe}[4]{%
  \begin{scope}[shift={#1}]
    \fill[white,rounded corners=4pt]
      (0,0) rectangle (#3,\PH);

    \begin{scope}
      \clip[rounded corners=4pt]
        (0,0) rectangle (#3,\PH);
      \fill[#2]
        (0,\PH-0.62) rectangle (#3,\PH);
    \end{scope}

    \draw[
      rounded corners=4pt,
      draw=#2,
      line width=1pt
    ]
      (0,0) rectangle (#3,\PH);

    \node[
      text=white,
      font=\bfseries
    ] at ({#3/2},\PH-0.31)
      {#4};
  \end{scope}%
}

\providecommand{\symrelaxok}{%
  \raisebox{-0.05ex}{%
    \normalfont\large\textbf{%
      \textcolor{esnnGreen}{\ensuremath{\checkmark}}%
    }%
  }%
}

\providecommand{\symrelaxno}{%
  \raisebox{-0.05ex}{%
    \normalfont\large\textbf{%
      \textcolor{badRed}{\ensuremath{\times}}%
    }%
  }%
}


\symrelaxpanelframe{(L)}{egnnBlue}{\PW}{Full $E(n)$ Equivariance}

\begin{scope}[shift={(L)}]

\node[
  font=\footnotesize\itshape,
  text=egnnBlue!80!black
] at (\PW/2,\PH-0.95)
  {No Preferred Direction};

\begin{scope}[shift={(\PW/2,3.8)}]

\def\Rr{1.6}

\shade[
  ball color=egnnBlue!10,
  opacity=0.4
]
  (0,0) circle (\Rr);

\draw[
  egnnBlue,
  thick
]
  (0,0) circle (\Rr);

\draw[
  egnnBlue!80!black,
  line width=1.2pt,
  -stealth
]
  (-\Rr,0)
  arc (180:360:{\Rr} and {0.35*\Rr});

\draw[
  egnnBlue!80!black,
  line width=0.8pt,
  dashed
]
  (\Rr,0)
  arc (0:180:{\Rr} and {0.35*\Rr});

\draw[
  egnnBlue!80!black,
  line width=1.2pt,
  -stealth
]
  (0,-\Rr)
  arc (270:450:{0.35*\Rr} and {\Rr});

\draw[
  egnnBlue!80!black,
  line width=0.8pt,
  dashed
]
  (0,\Rr)
  arc (90:270:{0.35*\Rr} and {\Rr});

\begin{scope}[rotate=45]

  \draw[
    egnnBlue!80!black,
    line width=1.2pt,
    -stealth
  ]
    (-\Rr,0)
    arc (180:360:{\Rr} and {0.35*\Rr});

  \draw[
    egnnBlue!80!black,
    line width=0.8pt,
    dashed
  ]
    (\Rr,0)
    arc (0:180:{\Rr} and {0.35*\Rr});

\end{scope}

\fill[fieldDark]
  (0,0) circle (2pt);

\node[
  fieldDark,
  anchor=north west,
  inner sep=4pt
] at (0,0)
  {$\mathbf{x}_i$};

\end{scope}

\node[
  align=center,
  font=\normalsize,
  text width=4.6cm
] at (\PW/2,0.85)
  {The full $O(n)$ action\\is enforced};

\end{scope}


\symrelaxpanelframe{(M)}{coordPurple}{\CW}{Adaptive Symmetry Relaxation}

\begin{scope}[shift={(M)}]

\node[
  font=\footnotesize\itshape,
  text=coordPurple!80!black
] at (\CW/2,\PH-0.95)
  {Controlled Symmetry Reduction};

\node[
  font=\normalsize,
  text=fieldDark,
  align=center,
  text width=4cm
] at (\CW/2,4.8)
  {Directional conditioning via\\preferred direction $\mathbf{g}$};

\begin{scope}[shift={(\CW/2,3.0)}]

\def\barW{1.5}

\shade[
  left color=egnnBlue,
  right color=esnnGreen,
  rounded corners=2pt
]
  (-\barW,-0.15)
  rectangle
  (\barW,0.15);


\draw[
  thick,
  egnnBlue!80!black
]
  (-\barW,-0.3)
  --
  (-\barW,0.3);

\node[
  font=\normalsize\bfseries,
  text=egnnBlue!80!black,
  anchor=south
] at (-\barW,0.4)
  {$\lambda=0$};

\node[
  font=\footnotesize,
  text=egnnBlue!80!black,
  anchor=north
] at (-\barW,-0.4)
  {Full $E(n)$};


\draw[
  thick,
  esnnGreen!80!black
]
  (\barW,-0.3)
  --
  (\barW,0.3);

\node[
  font=\normalsize\bfseries,
  text=esnnGreen!80!black,
  anchor=south
] at (\barW,0.4)
  {$\lambda\neq0$};

\node[
  font=\footnotesize,
  text=esnnGreen!80!black,
  anchor=north
] at (\barW,-0.4)
  {Reduced $E_{\mathbf g}(n)$};

\fill[coordPurple]
  (0,0) circle (4pt);

\draw[
  white,
  thick
]
  (0,0) circle (4pt);

\node[
  font=\normalsize\bfseries,
  text=coordPurple,
  anchor=south,
  align=center
] at (0,0.2)
  {Learnable $\lambda$};

\end{scope}

\node[
  align=center,
  font=\normalsize,
  text width=4.6cm
] at (\CW/2,0.85)
  {Directional influence vanishes\\continuously as $\lambda\to0$};

\end{scope}


\symrelaxpanelframe{(R)}{esnnGreen}{\PW}{Stabilizer $O_{\mathbf g}(n)$}

\begin{scope}[shift={(R)}]

\node[
  font=\footnotesize\itshape,
  text=esnnGreen!75!black
] at (\PW/2,\PH-0.95)
  {Fixed Ambient Direction};

\begin{scope}[shift={(\PW/2,3.5)}]

\def\Rr{1.25}
\def\ex{0.35*\Rr}
\def\HH{1.35}

\path[
  fill=esnnGreen!5,
  opacity=0.8
]
  (-\Rr,-\HH)
  rectangle
  (\Rr,\HH);

\path[
  fill=esnnGreen!15
]
  (0,\HH)
  ellipse ({\Rr} and {\ex});

\draw[
  esnnGreen!70!black,
  thick
]
  (-\Rr,-\HH)--(-\Rr,\HH);

\draw[
  esnnGreen!70!black,
  thick
]
  (\Rr,-\HH)--(\Rr,\HH);

\draw[
  esnnGreen!70!black,
  thick
]
  (0,\HH)
  ellipse ({\Rr} and {\ex});

\draw[
  esnnGreen!70!black,
  thick
]
  (-\Rr,-\HH)
  arc (180:360:{\Rr} and {\ex});

\draw[
  esnnGreen!70!black,
  thick,
  dashed
]
  (\Rr,-\HH)
  arc (0:180:{\Rr} and {\ex});

\draw[
  fieldDark,
  line width=2pt,
  -stealth
]
  (0,-\HH-0.3)
  --
  (0,\HH+0.1);

\node[
  font=\large\bfseries,
  text=fieldDark,
  anchor=west
] at (0.1,\HH-0.15)
  {$\mathbf{g}$};

\draw[
  esnnGreen!80!black,
  line width=1.5pt,
  -stealth
]
  (-\Rr,0)
  arc (180:360:{\Rr} and {\ex});

\draw[
  esnnGreen!80!black,
  line width=1pt,
  dashed
]
  (\Rr,0)
  arc (0:180:{\Rr} and {\ex});

\node[
  anchor=west
] at (\Rr+0.1,0)
  {\symrelaxok};

\draw[
  badRed,
  line width=1.2pt,
  dashed,
  -stealth
]
  (0,-\Rr)
  arc (270:450:{\ex} and {\Rr});

\draw[
  badRed,
  line width=0.8pt,
  dashed
]
  (0,\Rr)
  arc (90:270:{\ex} and {\Rr});

\node[
  anchor=south west
] at (\ex,\Rr-0.2)
  {\symrelaxno};

\fill[fieldDark]
  (0,0) circle (2pt);

\end{scope}

\node[
  align=center,
  font=\normalsize,
  text width=4.6cm
] at (\PW/2,0.85)
  {Only transformations satisfying\\$Q\mathbf{g}=\mathbf{g}$ are enforced};

\end{scope}

\end{tikzpicture}%
}

\caption{
\emph{Controlled symmetry relaxation through a preferred ambient direction.}
When all $\lambda_{\mathbf g}^{(\ell)}=0$, the directional pathway is inactive
and ESNN retains full $E(n)$-equivariance. Activating the directional
conditioning $\lambda_{\mathbf g}^{(\ell)} \langle\mathbf r_{ij},\mathbf g\rangle$ reduces the guaranteed symmetry to $E_{\mathbf g}(n)=O_{\mathbf g}(n)\ltimes\mathbb{R}^n$, where $O_{\mathbf g}(n)=\{Q\in O(n):Q\mathbf g=\mathbf g\}$. The layerwise relaxation coefficients are initialized at zero; $\mathbf g$ may be specified or learned, but is held fixed when the input geometry is transformed.
}

\label{fig:symmetry-relaxation}

\end{figure}

\begin{table}[htbp]
  \centering
  \footnotesize
  \setlength{\tabcolsep}{6pt}
  \renewcommand{\arraystretch}{0.95}

  \caption{\emph{N-body dynamics.}
    Mean squared error (MSE) for charged and gravity-augmented prediction.
    Charged ESNN and gravity results are mean $\pm$ standard deviation over five
    runs; baseline uncertainties are reported as in the original sources. For gravity,
    $\max_{\ell}|\lambda_g^{(\ell)}|\|\mathbf g\|_2$ measures directional-pathway activation and $A_g$ alignment with the gravity axis; for \emph{Fixed}, $A_g=1$ by construction. Lower MSE and higher $A_g$ are better.}
    \vspace{-0.2cm}
    \label{tab:nbody_mse}
  \label{tab:nbody_mse}

  \begin{subtable}[t]{0.35\linewidth}
    \centering
    \caption{\emph{N-body}}
    \label{tab:nbody}

    \setlength{\tabcolsep}{5pt}

    \resizebox{\linewidth}{!}{%
    \begin{tabular}{@{}lc@{}}
      \toprule
      \textbf{Method}
      & \textbf{MSE} $\downarrow$ \\
      \midrule

      SE(3)-Transformer
        & $0.0248 \pm 0.0010$ \\

      TFN
        & $0.0154 \pm 0.0023$ \\

      Radial Field
        & $0.0106 \pm 0.0001$ \\

      GMN
        & $0.0082 \pm 0.0003$ \\

      EGNN
        & $0.0076 \pm 0.0003$ \\

      GVP-GNN
        & $0.0069 \pm 0.0002$ \\

      FA-GNN
        & $0.0057 \pm 0.0002$ \\

      \midrule

      \textbf{ESNN-Id}
        & ${0.0084 \pm 0.0042}$ \\

      \textbf{ESNN-Diag}
        & {\textbf{$0.0053 \pm 0.0002$}} \\

      \textbf{ESNN-Ortho}
        & {\textbf{$0.0060 \pm 0.0003$}} \\

      \textbf{ESNN-RadTan}
        & {\textbf{$0.0057 \pm 0.0006$}} \\

      \bottomrule
    \end{tabular}%
    }
  \end{subtable}
  \hfill
  \begin{subtable}[t]{0.63\linewidth}
    \centering
    \caption{\emph{N-body + Gravity}}
    \label{tab:nbody_gravity}

    \setlength{\tabcolsep}{3pt}

    \resizebox{\linewidth}{!}{%
    \begin{tabular}{@{}llccc@{}}
      \toprule
      \textbf{Mode}
      & \textbf{Transport}
      & \textbf{MSE} $\downarrow$
      & $\boldsymbol{\max_{\ell}|\lambda_g^{(\ell)}|\|\mathbf g\|_2}$
      & $\boldsymbol{A_g}$ $\uparrow$ \\
      \midrule

      \emph{None}
        & ESNN-Diag
        & $0.1072 \pm 0.0346$
        & --
        & -- \\

      & ESNN-Ortho
        & $0.1252 \pm 0.0490$
        & --
        & -- \\

      & ESNN-RadTan
        & $0.1019 \pm 0.0324$
        & --
        & -- \\

      \midrule

      \emph{Learned}
        & ESNN-Diag
        & $0.0205 \pm 0.0016$
        & $0.2674 \pm 0.0704$
        & $1.000$ \\

      & ESNN-Ortho
        & $0.0197 \pm 0.0015$
        & $0.4009 \pm 0.0722$
        & $1.000$ \\

      & ESNN-RadTan
        & $0.0213 \pm 0.0033$
        & $2.0552 \pm 0.6337$
        & $1.000$ \\

      \midrule

      \emph{Fixed}
        & ESNN-Diag
        & $0.0208 \pm 0.0026$
        & $0.3464 \pm 0.0357$
        & $1.000$ \\

      & ESNN-Ortho
        & $0.0198 \pm 0.0017$
        & $0.6100 \pm 0.1771$
        & $1.000$ \\

      & ESNN-RadTan
        & $0.0234 \pm 0.0054$
        & $1.3792 \pm 0.3014$
        & $1.000$ \\

      \bottomrule
    \end{tabular}%
    }

  \end{subtable}
  \vspace{-0.5cm}
\end{table}

\begin{theorem}[Stabilizer Subequivariance]
\label{thm:cylindrical}
Let $\mathbf{g}\neq\mathbf{0}$ be held fixed under transformations of the
input geometry, and define its stabilizer:
\begin{equation}
    O_{\mathbf{g}}(n)
    =
    \left\{
        Q\in O(n)
        \;:\;
        Q\mathbf{g}=\mathbf{g}
    \right\}
    \label{eq:stabilizer_group}
\end{equation}
An ESNN conditioned on Equation~\ref{eq:relaxed_embed} is equivariant under
translations and every $Q\in O_{\mathbf{g}}(n)$, with guaranteed symmetry
group:
\begin{equation}
    E_{\mathbf{g}}(n)
    =
    O_{\mathbf{g}}(n)\ltimes\mathbb{R}^{n}
\end{equation}
When $\lambda_{\mathbf{g}}^{(\ell)}=0$ for every layer, the directional
conditioning vanishes and full $E(n)$-equivariance is recovered.
\end{theorem}

For any $Q\in O_{\mathbf{g}}(n)$, we have $
    \left\langle
        Q\mathbf{r}_{ij},\mathbf{g}
    \right\rangle
    =
    \left\langle
        \mathbf{r}_{ij},Q^\top\mathbf{g}
    \right\rangle
    =
    \left\langle
        \mathbf{r}_{ij},\mathbf{g}
    \right\rangle$, so the directional term is invariant under the stabilizer of
$\mathbf{g}$. Together with the covariant transport maps of
Section~\ref{sec:transports}, this gives the subgroup equivariance.
The full proof is provided in Appendix~\ref{app:proof_thm4}.

\noindent\textbf{Symmetry-Preserving Initialization.}
The coefficients $\lambda_{\mathbf{g}}^{(\ell)}$ are initialized at zero, so
the directional conditioning is inactive at initialization and the model
begins in the fully $E(n)$-equivariant regime. During training, these coefficients can become nonzero when the data support
directional dependence. The preferred direction $\mathbf{g}$ can either be specified from known physical structure or learned jointly with the rest of the model.

\section{Experiments}
\label{sec:exp}

We evaluate ESNN through four questions: \textbf{Q1} whether richer equivariant
transport improves dynamics prediction when full $E(3)$ symmetry is appropriate;
\textbf{Q2} whether ESNN can exploit or recover a preferred ambient direction;
\textbf{Q3} how the transport mechanisms perform on irregular mesh-based
physical systems; and \textbf{Q4} whether the same framework remains robust to
unseen rotations in point-cloud classification. We consider charged and
gravity-augmented N-body dynamics~\citep{kipf2018neural}, three MeshGraphNets
benchmarks~\citep{pfaff2021learning}, and ModelNet40~\citep{wu20153d}. Appendix~\ref{app:qm9_details} additionally evaluates invariant QM9 targets,
where the best ESNN transport improves over EGNN on ten of twelve properties.
Full experimental details are provided in
Appendix~\ref{app:experimental_details}.

\noindent\textbf{Q1: Particle Dynamics --- Charged N-Body Dynamics.}
We first consider the standard five-particle charged N-body benchmark, where
the model predicts future coordinates from positions, velocities, and charges.
Table~\ref{tab:nbody_mse}(a) shows that all three learned transports achieve
lower mean error than EGNN, whereas ESNN-Id is less accurate and substantially
more variable. 

\begin{wrapfigure}{r}{0.56\columnwidth}
    \vspace{-0.7em}
    \centering

    \captionof{table}{\emph{ModelNet40 rotation generalization.}
    Accuracy (\%) under $z/z$, $z/\mathrm{SO}(3)$, and
    $\mathrm{SO}(3)/\mathrm{SO}(3)$. The $z/\mathrm{SO}(3)$ result evaluates the
    same $z$-trained checkpoint under unseen $\mathrm{SO}(3)$ rotations.
    $\Delta_{\mathrm{OOD}}$ is the absolute $z/z$--$z/\mathrm{SO}(3)$ accuracy
    change in percentage points. Literature baselines follow
    \citet{lippmann2024beyond}; EGNN and ESNN use the same evaluation protocol.
    Higher accuracy and lower $\Delta_{\mathrm{OOD}}$ are better.}
    \label{tab:modelnet40_results}
    \label{tab:modelnet40_results}

    \scriptsize
    \setlength{\tabcolsep}{2.3pt}
    \renewcommand{\arraystretch}{0.78}

    \resizebox{\linewidth}{!}{%
    \begin{tabular}{@{}lcccc@{}}
        \toprule
        \textbf{Method}
        & $\boldsymbol{z/z}$
        & $\boldsymbol{z/\mathrm{SO}(3)}$
        & $\boldsymbol{\mathrm{SO}(3)/\mathrm{SO}(3)}$
        & $\boldsymbol{\Delta_{\mathrm{OOD}}}$ $\downarrow$ \\
        \midrule

        PointNet
        & 85.9 & 19.6 & 74.7 & 66.3 \\

        RS-CNN
        & 90.3 & 48.7 & 82.6 & 41.6 \\

        DGCNN
        & 90.3 & 33.8 & 88.6 & 56.5 \\

        RI-Conv
        & 86.5 & 86.4 & 86.4 & 0.1 \\

        GC-Conv
        & 89.0 & 89.1 & 89.2 & 0.1 \\

        Luo et al.\ DGCNN
        & 88.4 & 88.4 & 88.9 & 0.0 \\

        LGR-Net
        & 90.9 & 90.9 & 91.1 & 0.0 \\

        Li et al.\ (w/ TTA)
        & 91.6 & 91.6 & 91.6 & 0.0 \\

        CRIN
        & 91.8 & 91.8 & 91.8 & 0.0 \\

        TFN
        & 88.5 & 85.3 & 87.6 & 3.2 \\

        VN-PointNet
        & 77.5 & 77.5 & 77.2 & 0.0 \\

        VN-DGCNN
        & 89.5 & 89.5 & 90.2 & 0.0 \\

        \midrule

        \textbf{EGNN}
        & 57.6 & 66.6 & 62.7 & 9.0 \\

        \textbf{ESNN-Id}
        & 84.7 & 85.7 & 85.6 & 1.1 \\

        \textbf{ESNN-Diag}
        & 84.6 & 84.3 & 84.4 & 0.3 \\

        \textbf{ESNN-Ortho}
        & 84.9 & 85.2 & 84.7 & 0.2 \\

        \textbf{ESNN-RadTan}
        & 85.4 & 84.6 & 86.3 & 0.7 \\

        \bottomrule
    \end{tabular}%
    }

    \vspace{-0.7em}
\end{wrapfigure} ESNN-Diag achieves the lowest ESNN error at
$0.0053\pm0.0002$, approximately $30\%$ below EGNN
($0.0076\pm0.0003$), followed by ESNN-RadTan at $0.0057\pm0.0006$ and
ESNN-Ortho at $0.0060\pm0.0003$. The learned transport variants outperform both EGNN and the Identity reference in this setting, supporting the benefit of richer edge-wise geometric interactions under full Euclidean symmetry.

\noindent\textbf{Q2: Particle Dynamics --- Gravity N-Body Dynamics.}
We next introduce a uniform gravitational field,
$\mathbf{a}_{\mathbf g}=(0,0,-9.81)^\top$, with preferred axis
$\widehat{\mathbf g}_{\mathrm{true}}=(0,0,-1)^\top$. We compare the fully
$E(3)$-equivariant \emph{None} model with two symmetry-relaxed variants:
\emph{Fixed}, which is given the gravity axis, and \emph{Learned}, which
estimates it from data using the mechanism of Section~\ref{sec:soft_symmetry}.
Table~\ref{tab:nbody_mse}(b) shows that both variants reduce MSE from
approximately $0.10$--$0.13$ to about $0.02$ across all three transports.
The \emph{Learned} models achieve errors comparable to \emph{Fixed} while
their inferred directions attain sign-invariant alignment $A_g$ that rounds
to $1.000$ in every case. ESNN therefore recovers the relevant
symmetry-breaking axis from data without requiring it to be specified
explicitly.

\noindent\textbf{Q3: Mesh-Based Physical Dynamics.}
We evaluate ESNN on \textsc{CylinderFlow}, \textsc{DeformingPlate}, and
\textsc{Airfoil} from MeshGraphNets~\citep{pfaff2021learning}, covering fluid
flow, structural deformation, and compressible aerodynamics on irregular
meshes. We report RMSE for one-step prediction, 50-step autoregressive rollout,
and the full trajectory; evaluation details are provided in
Appendix~\ref{app:meshgraphnets_details}.
Table~\ref{tab:meshgraphnets_results} shows task- and horizon-dependent
performance. On \textsc{DeformingPlate}, ESNN-RadTan achieves the lowest error
at every horizon, reducing full-rollout RMSE from $13.6$ for MeshGraphNet to
$5.8$. On \textsc{CylinderFlow}, learned transports improve one- and 50-step
prediction, while MeshGraphNet remains strongest over the full trajectory. On
\textsc{Airfoil}, MeshGraphNet remains best at all three horizons. Across the
three tasks, ESNN consistently improves over EGNN, with the largest gains on
\textsc{DeformingPlate}.

\begin{table*}[t]
    \centering
    \scriptsize
    \setlength{\tabcolsep}{4.0pt}
    \renewcommand{\arraystretch}{1.0}

    \caption{\emph{Mesh-based physical dynamics.}
    RMSE ($\times10^{-3}$) for one-step, 50-step, and full autoregressive
    rollouts on \textsc{CylinderFlow}, \textsc{DeformingPlate}, and
    \textsc{Airfoil}. Results are mean $\pm$ standard deviation over five matched runs. Best values are bold; \emph{Diverged} denotes an unstable autoregressive rollout. Lower is better.}
    \label{tab:meshgraphnets_results}

    \resizebox{\textwidth}{!}{%
    \begin{tabular}{@{}lccc|ccc|ccc@{}}
        \toprule

        & \multicolumn{3}{c}{\textbf{CylinderFlow}}
        & \multicolumn{3}{c}{\textbf{DeformingPlate}}
        & \multicolumn{3}{c}{\textbf{Airfoil}} \\

        \cmidrule(lr){2-4}
        \cmidrule(lr){5-7}
        \cmidrule(lr){8-10}

        \textbf{Method}
        & \textbf{1-Step}
        & \textbf{50-Step}
        & \textbf{Full}
        & \textbf{1-Step}
        & \textbf{50-Step}
        & \textbf{Full}
        & \textbf{1-Step}
        & \textbf{50-Step}
        & \textbf{Full} \\

        \midrule

        MeshGraphNet
        & $4.05 \pm 0.08$
        & $12.5 \pm 0.6$
        & $\mathbf{46.14 \pm 6.3}$
        & $0.15 \pm 0.02$
        & $2.2 \pm 0.5$
        & $13.6 \pm 1.4$
        & $1141 \pm 68$
        & $1529 \pm 68$
        & $\mathbf{8033 \pm 662}$ \\

        \midrule

        EGNN
        & $9.20 \pm 1.80$
        & $40.5 \pm 2.5$
        & $158.09 \pm 2.5$
        & $0.35 \pm 0.01$
        & $1.6 \pm 0.1$
        & $11.5 \pm 0.5$
        & $5001 \pm 134$
        & $18948 \pm 609$
        & diverged \\

        ESNN-Id
        & $5.35 \pm 0.32$
        & $20.6 \pm 1.0$
        & $84.09 \pm 3.4$
        & $0.15 \pm 0.01$
        & $1.3 \pm 0.1$
        & $7.7 \pm 0.4$
        & $4744 \pm 154$
        & $7579 \pm 287$
        & $16893 \pm 740$ \\

        ESNN-Diag
        & $2.77 \pm 0.17$
        & $8.8 \pm 0.4$
        & $50.84 \pm 2.0$
        & $0.17 \pm 0.01$
        & $1.4 \pm 0.1$
        & $9.2 \pm 0.5$
        & $2413 \pm 49$
        & $4966 \pm 126$
        & $11014 \pm 321$ \\

        ESNN-Ortho
        & $2.56 \pm 0.15$
        & $8.5 \pm 0.4$
        & $55.42 \pm 2.2$
        & $0.13 \pm 0.01$
        & $1.1 \pm 0.1$
        & $8.1 \pm 0.4$
        & $1459 \pm 27$
        & $2807 \pm 66$
        & $8821 \pm 259$ \\

        ESNN-RadTan
        & $3.16 \pm 0.19$
        & $10.2 \pm 0.5$
        & $49.79 \pm 2.0$
        & $0.08 \pm 0.01$
        & $1.0 \pm 0.1$
        & $\mathbf{5.8 \pm 0.3}$
        & $1540 \pm 29$
        & $2386 \pm 59$
        & $8600 \pm 259$ \\

        \bottomrule
    \end{tabular}%
    }
\end{table*}

\noindent\textbf{Q4: Rotation Generalization --- ModelNet40.}
We evaluate whether the transport framework extends beyond physical simulation
to point-cloud classification. ModelNet40 shapes are represented as local
$k$-nearest-neighbor graphs~\citep{wu20153d}, and we consider the $z/z$,
$z/\mathrm{SO}(3)$, and $\mathrm{SO}(3)/\mathrm{SO}(3)$ protocols. The
$z/\mathrm{SO}(3)$ setting evaluates the same $z$-trained checkpoint under
unseen three-dimensional rotations. Table~\ref{tab:modelnet40_results} shows
that ESNN remains stable under this shift, with $\Delta_{\mathrm{OOD}}$ between
$0.2$ and $1.1$ percentage points across transport variants, compared with
$9.0$ points for EGNN. Specialized rotation-robust architectures achieve higher
absolute accuracy, while ESNN retains approximately $84$--$86\%$ accuracy
across all three protocols. These results extend ESNN to point-cloud
classification while preserving robustness to unseen global rotations.

\vspace{-0.2cm}
\section{Conclusions and Limitations}
\vspace{-0.1cm}
\label{sec:conclusion}

We introduced ESNN, a first-order Cartesian framework that enriches equivariant
message passing through structured matrix-valued transport of vector features
across graph edges. ESNN increases geometric expressivity in the edge
interactions without requiring higher-order representations. We characterize
displacement-conditioned linear $O(n)$-equivariant transport and show how
learned covariant features enable richer feature-conditioned transformations.
We also introduce controlled symmetry relaxation for systems with a preferred
ambient direction, preserving exact equivariance to the remaining symmetry
while recovering full $E(n)$-equivariance as a special case. Across diverse
geometric learning problems, ESNN improves predictive performance over
first-order equivariant baselines, learns a preferred direction from data, and
remains robust to unseen rotations. The current formulation is limited to
scalar and first-order vector features; the completeness result applies only to
displacement-conditioned linear transport, and symmetry relaxation assumes a
single global preferred direction. Future work could extend ESNN to richer
representations, local or gauge-equivariant transport, and more general
symmetry-breaking fields.






\bibliography{references}
\bibliographystyle{iclr2027_conference}

\newpage
\clearpage
\newpage
\appendix
%

\providecolor{Green}{RGB}{0,128,0}
\providecolor{Red}{RGB}{200,0,0}
\providecommand{\cmark}{\textcolor{Green}{\ding{51}}}
\providecommand{\xmark}{\textcolor{Red}{\ding{55}}}

\providecolor{Top1}{HTML}{0072B2}
\providecolor{Top2}{HTML}{E69F00}
\providecolor{Top3}{HTML}{CC79A7}
\providecolor{ufill}{HTML}{E1D5E7}
\providecolor{uborder}{HTML}{9673A6}
\providecolor{vfill}{HTML}{D5E8D4}
\providecolor{vborder}{HTML}{82B366}

\sisetup{
    mode=text,
    table-alignment-mode=none,
    reset-text-family=false,
    reset-text-series=false,
    reset-text-shape=false,
    table-number-alignment=center,
    table-align-comparator=false,
    table-align-text-pre=false,
    table-align-text-post=false,
    round-mode=uncertainty,
    round-precision=3
}

\providecolor{proofbar}{RGB}{225,225,225}
\providecolor{theorembar}{RGB}{128,0,128}
\providecolor{propositionbar}{RGB}{0,100,0}
\providecolor{definitionbar}{RGB}{0,0,200}
\providecolor{lemmabar}{RGB}{255,140,0}
\providecolor{corollarybar}{RGB}{204,153,255}
\providecolor{remarkbar}{RGB}{0,139,139}
\providecolor{assumptionbar}{RGB}{178,34,34}

%
%

\newenvironment{restatementframe}[1]{%
  \ificlrfinal
  \begin{mdframed}[
    hidealllines=true,
    leftline=true,
    linecolor=#1,
    linewidth=2pt,
    innerleftmargin=6pt,
    innerrightmargin=0pt,
    innertopmargin=2pt,
    innerbottommargin=2pt
  ]%
  \else
    \par\addvspace{\topsep}\noindent
  \fi
}{%
  \ificlrfinal
    \end{mdframed}%
  \else
    \par\addvspace{\topsep}%
  \fi
}

\newenvironment{restatedtheorem}[2][]{%
  \begin{restatementframe}{theorembar}%
  \noindent\textbf{Theorem~\ref{#2}}%
  \if\relax\detokenize{#1}\relax
    \textbf{.}%
  \else
    \textbf{ (#1).}%
  \fi
  \itshape\ignorespaces
}{%
  \end{restatementframe}
}

\newenvironment{restatedproposition}[2][]{%
  \begin{restatementframe}{propositionbar}%
  \noindent\textbf{Proposition~\ref{#2}}%
  \if\relax\detokenize{#1}\relax
    \textbf{.}%
  \else
    \textbf{ (#1).}%
  \fi
  \itshape\ignorespaces
}{%
  \end{restatementframe}
}

\newenvironment{restatedlemma}[2][]{%
  \begin{restatementframe}{lemmabar}%
  \noindent\textbf{Lemma~\ref{#2}}%
  \if\relax\detokenize{#1}\relax
    \textbf{.}%
  \else
    \textbf{ (#1).}%
  \fi
  \itshape\ignorespaces
}{%
  \end{restatementframe}
}

\newenvironment{restatedcorollary}[2][]{%
  \begin{restatementframe}{corollarybar}%
  \noindent\textbf{Corollary~\ref{#2}}%
  \if\relax\detokenize{#1}\relax
    \textbf{.}%
  \else
    \textbf{ (#1).}%
  \fi
  \itshape\ignorespaces
}{%
  \end{restatementframe}
}

\providecommand*{\theoremautorefname}{Theorem}
\providecommand*{\lemmaautorefname}{Lemma}
\providecommand*{\propositionautorefname}{Proposition}
\providecommand*{\definitionautorefname}{Definition}
\providecommand*{\corollaryautorefname}{Corollary}
\providecommand*{\remarkautorefname}{Remark}
\providecommand*{\exampleautorefname}{Example}
\providecommand*{\assumptionautorefname}{Assumption}

\makeatletter
\@ifundefined{ifshowcomments}{\newboolean{showcomments}}{}
\makeatother
\setboolean{showcomments}{true}
\providecommand{\fr}[1]{\textcolor{red}{\textbf{[FR: #1]}}}

\begin{center}
    {\LARGE \bfseries Equivariant Sheaf Neural Networks: Learning Geometric Transport on Graphs \par}
    \vspace{0.4em}
    {\Large Supplementary Material\par}
\end{center}

\vspace{1em}

\addcontentsline{toc}{part}{Supplementary Material}
\etocsetnexttocdepth{3}
\vspace{-1cm}
\localtableofcontents

\vspace{1em}
\section{Extended Related Work}
\label{app:related}

This appendix extends the discussion in Section~\ref{sec:back} by placing ESNN within three closely related areas: equivariant geometric learning, sheaf-based representation learning, and methods that relax or reduce symmetry in the
presence of external structure. We focus on approaches that act on geometric vector features, learn edge-dependent transformations between local representations, or adapt the symmetry imposed on message passing.

\subsection{Cartesian and Steerable Equivariant Models}
\label{app:related_equivariant_gnns}

A central distinction among $E(n)$-equivariant graph networks lies in the
geometric representations they propagate and how these representations are
coupled across neighboring nodes. One broad family works directly with
invariant scalars and low-order Cartesian features. EGNN
~\citep{satorras2022enequivariantgraphneural} constructs coordinate updates
from invariant edge functions acting on relative displacement vectors, while
PaiNN~\citep{schutt2021equivariant} and GVP-based models
~\citep{jing2021equivariant} maintain scalar and vector channels coupled
through equivariant operations. Subsequent architectures have expanded this Cartesian design space through richer scalar--vector interactions~\citep{wang2022visnet}, multiple vector
channels~\citep{levy2023multiple}, and Cartesian tensor features
~\citep{simeon2306tensornet}. More recent models
such as GotenNet~\citep{aykent2025gotennet} similarly seek greater geometric
expressivity while avoiding explicit Clebsch--Gordan contractions. A complementary family uses steerable features transforming under higher-order group representations. Tensor Field Networks ~\citep{thomas2018tensor}, SE(3)-Transformers~\citep{fuchs2020se},
SEGNN~\citep{brandstetter2021geometric}, MACE~\citep{batatia2022mace}, and
EquiformerV2~\citep{liao2024equiformerv2} represent angular structure through
spherical harmonics and tensor-product or related steerable operations.
Higher representation degrees can capture increasingly rich angular
information, but also increase the cost of the representations and their
couplings. Recent work has therefore investigated how much high-degree
information is needed and how it can be processed more efficiently
~\citep{cen2024are}. ESNN instead retains a first-order scalar--vector representation and increases geometric expressivity through matrix-valued transport of vector features across graph edges.

\subsection{Frames, Gauges, and Topological Domains}

Another route to equivariance is to resolve orientation dependence through
the choice or averaging of reference frames. Frame Averaging ~\citep{puny2022frame} constructs exactly equivariant models by averaging a
backbone over an equivariant frame, with FAENet~\citep{duval2023faenet}
adapting this idea to atomistic systems. Learned canonicalization
~\citep{kaba2023equivariance} instead predicts a canonical representative
before applying a generic backbone, while subsequent work has examined the
benefit of retaining tensorial messages beyond canonicalized scalar
representations~\citep{lippmann2024beyond}. Local gauge equivariance addresses a related but distinct setting, in which different sites may use independently chosen reference frames. Gauge-equivariant CNNs~\citep{cohen2019gauge} and coordinate-independent convolutions~\citep{weiler2021coordinate} require the network to transform consistently under changes of these local frames. Tangent-bundle convolutional networks connect this viewpoint to connection Laplacians and cellular sheaves through vector-field diffusion on Riemannian manifolds
~\citep{battiloro2023tangent}, while Torsor CNNs extend local group-valued
transport to arbitrary graphs through edge potentials relating neighboring
frames~\citep{li2025learning}. These models concern local gauge equivariance,
whereas ESNN enforces equivariance under a common ambient $O(n)$ action. Its
Orthogonal Transport therefore has a connection-like geometric interpretation,
but is not a general gauge-equivariant construction. A separate line of work extends Euclidean equivariance to richer combinatorial domains. $E(n)$-Equivariant Topological Neural Networks ~\citep{battiloro2024equivariant} generalize equivariant message passing from graphs to combinatorial complexes, while $E(n)$-equivariant message-passing cellular networks~\citep{kovac2024equivariant} develop related constructions
on cellular structures. These approaches enrich the underlying domain through
higher-order cells. ESNN is complementary: it remains graph-based and instead
enriches the transport between geometric feature spaces. A related theoretical perspective is provided by \citep{maruyama2026foundations}, who formulates order-equivariant neural networks on equivariant vector bundles
over face posets and recovers graph and sheaf layers within a common framework. Their construction organizes equivariance over the combinatorial structure of the domain, whereas ESNN focuses on the continuous Euclidean action on geometric feature fibers and
constrains edge transport accordingly. The two perspectives therefore address
distinct but compatible forms of symmetry.

\subsection{Sheaf Learning and Directionality}
\label{app:related_sheaf_nns}

Cellular sheaves associate local vector spaces with graph cells and relate
them through incidence restriction maps. Early SNNs
~\citep{hansen2020sheaf} used the resulting sheaf Laplacian to generalize
graph convolution, while Neural Sheaf Diffusion~\citep{bodnar2022neural}
introduced learnable restriction maps and studied their effect on heterophily,
class separation, and oversmoothing. These works established neural sheaf
diffusion as a framework for propagation through learned compatibility maps
between local feature spaces. Subsequent architectures have explored richer sheaf structures and propagation operators. Connection-Laplacian SNNs~\citep{barbero2022sheaf} use orthogonal restriction maps motivated by local tangent-space alignment, providing a particularly relevant precedent for geometric transport. Sheaf Attention Networks~\citep{barbero2022sheafattention} incorporate attention into sheaf propagation, Sheaf Hypergraph Networks~\citep{duta2023sheaf} extend the construction to higher-order relations, and Heterogeneous Sheaf Neural Networks~\citep{braithwaite2024heterogeneous} introduce heterogeneous stalk
and restriction structures for typed graphs. Polynomial Neural Sheaf
Diffusion~\citep{borgi2025polynomial} develops higher-order spectral filters
of the sheaf operator, while recent work has also considered non-Euclidean
stalk geometries such as SPD-valued representations~\citep{peng2026spd}. Directionality has become an important extension of this framework. Classical
sheaf Laplacians are self-adjoint and therefore do not directly represent
independent propagation rules for the two orientations of an edge.
Cooperative Sheaf Neural Networks~\citep{ribeiro2025cooperative} introduce
directed sheaf constructions with in- and out-degree Laplacians, while
Directed Sheaf Neural Networks~\citep{fiorini2025sheaves} develop a directed
cellular-sheaf formulation and a corresponding directional Laplacian. These
models show how asymmetric information flow can be incorporated without
forcing opposite edge orientations to encode the same interaction. A related
categorical perspective is provided by Copresheaf Topological Neural Networks
~\citep{hajij2025copresheaf}, which organize directed maps between local
feature spaces and their compositions. ESNN builds on this matrix-valued view of information transport while imposing the additional Euclidean transformation laws required by geometric vector features. As detailed in Appendix~\ref{app:sheaf_transport}, classical restriction maps induce node-to-node couplings such as $\rho_{i\to e}^{*}\rho_{j\to e}$, whereas ESNN parameterizes the directed node-to-node transport directly. Adjoint-consistent transports together with symmetric aggregation weights yield a self-adjoint specialization, while additionally imposing full-fiber orthogonality gives the exact connection-sheaf case. The general ESNN operator instead remains a directed matrix-valued transport between neighboring feature spaces. Recent empirical work also suggests that unrestricted learned sheaf maps are
not universally advantageous: identity-sheaf baselines can be competitive
with learned operators on standard heterophilic benchmarks
~\citep{caralt2026necessity}. ESNN therefore targets a more structured setting: geometric vector features with a known Euclidean transformation law, for which edge-wise transport can be designed to remain compatible with that symmetry.

\subsection{Symmetry Relaxation and Subequivariant Inductive Biases}
\label{app:related_symmetry_breaking}

External fields and environmental structure can reduce the symmetry of a
physical system. One approach is to enforce the subgroup that preserves the
relevant external direction. Subequivariant Graph Neural Networks
~\citep{han2022learning}, for example, incorporate fields such as gravity and
construct message-passing operations equivariant to the corresponding
stabilizer subgroup. ESNN follows the same group-theoretic principle, but
introduces the preferred direction within the transport framework and scales
its contribution through a learnable coefficient. A different line of work relaxes equivariance rather than replacing it with an
exact subgroup symmetry. Non-stationary continuous filters
\citep{wang2022relax}, approximately equivariant networks
~\citep{wang2022approximately}, relaxed group convolutions
~\citep{wang2024discovering}, and Relaxed Equivariant Graph Neural Networks
~\citep{hofgard2024relaxed} learn controlled departures from a prescribed
symmetry. External effects can also be modeled separately, as in Latent Field
Discovery~\citep{kofinas2023latent}, which augments equivariant local
interactions with an additional global neural field. ESNN is closest to the subequivariant setting: when directional conditioning is active, it preserves exact equivariance to the stabilizer of the preferred direction, while full $E(n)$-equivariance is recovered when the directional coefficients vanish across all layers.


\section{Algebraic Foundations of Equivariant Spatial Transport}
\label{app:canonical_derivations}

This appendix develops the algebraic structure underlying the transport maps
introduced in Section~\ref{sec:canonical_form}. We first derive their
spatial--channel factorization, then characterize the complete class of
$O(n)$-equivariant transports available when relative displacement is the only
covariant geometric input, and finally show how additional learned covariant
features give rise to richer spatial operators. Throughout, the transport is
linear with respect to the vector message for a fixed set of
context-dependent coefficients. Since these coefficients may themselves be
predicted from the current feature state, the resulting message function need
not be linear in the evolving features.

\subsection{Tensor-Product Structure of Vector Transport}
\label{app:tensor_decomposition}

The vector component of an ESNN stalk is:
\begin{equation}
    \mathcal{F}_{\mathrm{vec}}(i)
    =
    \mathbb{R}^{n}\otimes\mathbb{R}^{c_v}
\end{equation}
where $\mathbb{R}^{n}$ is the spatial factor carrying the $O(n)$ action and
$\mathbb{R}^{c_v}$ indexes the vector channels. After identifying the tensor
product with $\mathbb{R}^{n\times c_v}$, the group therefore acts as:
\begin{equation}
    \mathbf{V}
    \mapsto
    Q\mathbf{V},
    \qquad
    Q\in O(n)
\end{equation}

This separation between spatial and channel factors also carries over to
linear operators. In finite dimensions:
\begin{equation}
    \operatorname{End}
    \left(
        \mathbb{R}^{n}\otimes\mathbb{R}^{c_v}
    \right)
    \cong
    \operatorname{End}(\mathbb{R}^{n})
    \otimes
    \operatorname{End}(\mathbb{R}^{c_v})
    \label{eq:end_tensor_decomposition}
\end{equation}
so every linear operator on the vector stalk can be expressed as a finite
sum of separable spatial and channel transformations.

\begin{proposition}[Separable Expansion of Linear Vector Transport]
\label{prop:separable_transport}
Let
$\mathcal{L}:
\mathbb{R}^{n\times c_v}
\longrightarrow
\mathbb{R}^{n\times c_v}$
be linear. Then there exist spatial matrices
$\mathbf{S}^{(k)}\in\mathbb{R}^{n\times n}$ and channel matrices
$\mathbf{M}^{(k)}\in\mathbb{R}^{c_v\times c_v}$ such that:
\begin{equation}
    \mathcal{L}(\mathbf{V})
    =
    \sum_{k=1}^{K}
    \mathbf{S}^{(k)}
    \mathbf{V}
    \mathbf{M}^{(k)}
    \label{eq:appendix_separable_transport}
\end{equation}
for some finite $K$.
\end{proposition}

\begin{proof}
Equation~\ref{eq:end_tensor_decomposition} implies that any endomorphism of
$\mathbb{R}^{n}\otimes\mathbb{R}^{c_v}$ can be written as a finite sum of
elementary tensor-product operators. Under the matrix identification
$\mathbb{R}^{n}\otimes\mathbb{R}^{c_v}
\simeq
\mathbb{R}^{n\times c_v}$, the standard vectorization identity:
\begin{equation}
    \operatorname{vec}
    \left(
        \mathbf{S}\mathbf{V}\mathbf{M}
    \right)
    =
    \left(
        \mathbf{M}^{\top}\otimes\mathbf{S}
    \right)
    \operatorname{vec}(\mathbf{V})
\end{equation}
maps each elementary tensor-product operator to a left spatial action and a
right channel action, yielding
Equation~\ref{eq:appendix_separable_transport}.
\end{proof}

Proposition~\ref{prop:separable_transport} is purely algebraic and does not
by itself enforce equivariance. In ESNN, the spatial and channel matrices may
depend on the local edge context, so this dependence must itself respect the
group action. For a transformed context
$Q\!\cdot\!\mathcal{C}_{ij}$, the sufficient component-wise conditions
introduced in Section~\ref{sec:canonical_form} are:
\begin{equation}
    \mathbf{S}^{(k)}_{ij}
    (Q\!\cdot\!\mathcal{C}_{ij})
    =
    Q
    \mathbf{S}^{(k)}_{ij}(\mathcal{C}_{ij})
    Q^\top
    \qquad
    \mathbf{M}^{(k)}_{ij}
    (Q\!\cdot\!\mathcal{C}_{ij})
    =
    \mathbf{M}^{(k)}_{ij}(\mathcal{C}_{ij})
    \label{eq:appendix_transport_constraints}
\end{equation}
These are precisely the conditions used in
Proposition~\ref{prop:transport_equivariance}. The decomposition separates
the spatial and channel actions, while the corresponding covariance and
invariance constraints ensure compatibility with the ambient $O(n)$ action. Figure~\ref{fig:bilinear_transport} illustrates this spatial--channel
factorization.

\begin{figure}[t]
\centering
  \resizebox{\textwidth}{!}{
%
%
%

\begin{tikzpicture}[
  font=\large,
  line cap=round,
  line join=round
]

%
%


\def\nRows{4}
\def\nCols{5}
\def\cs{0.56}

\pgfmathsetmacro{\Vw}{\nCols*\cs}
\pgfmathsetmacro{\Vh}{\nRows*\cs}

\coordinate (Vc) at (0,0);
\coordinate (Vsw) at ($(Vc)+(-\Vw/2,-\Vh/2)$);


%
\providecommand{\bilinearChanBars}[1]{%
  \begin{scope}
    \foreach \i/\h in {
      0/0.30,
      1/0.52,
      2/0.20,
      3/0.44,
      4/0.34
    }{%
      \fill[#1]
        (\i*0.16,0)
        rectangle
        (\i*0.16+0.11,\h);
    }
  \end{scope}%
}


\begin{scope}

  \def\TW{18.4}

  \coordinate (TBL) at (-\TW/2,\Vh/2+4.05);

  \draw[
    rounded corners=4pt,
    fill=esnnGreen,
    draw=esnnGreen,
    line width=1pt
  ]
    (TBL)
    rectangle
    ($(TBL)+(\TW,0.74)$);

  \node[
    text=white,
    font=\bfseries
  ] at ($(TBL)+(\TW/2,0.37)$)
    {Canonical ESNN Transport};

\end{scope}


\node[
  opbox,
  line width=0.9pt,
  inner sep=6pt,
  font=\large,
  align=center
] at (0,\Vh/2+2.55)
  {%
    $\displaystyle
      \mathcal T_{i\leftarrow j}(\mathbf V_j)
      =
      \sum_{k=1}^{K}
      \textcolor{coordPurple}{
        \mathbf S^{(k)}_{ij}
      }
      \mathbf V_j
      \textcolor{scalarSlate}{
        \mathbf M^{(k)}_{ij}
      }
    $
  };


\begin{scope}[shift={(Vsw)}]


  \foreach \r in {0,...,\numexpr\nRows-1\relax}{%
    \foreach \c in {0,...,\numexpr\nCols-1\relax}{%

      \pgfmathsetmacro{\tint}{13+7*mod(\r+\c,2)}

      \fill[vectorTeal!\tint]
        (\c*\cs,\r*\cs)
        rectangle
        (\c*\cs+\cs,\r*\cs+\cs);
    }
  }

  \draw[
    vectorTeal!50,
    line width=0.4pt,
    xstep=\cs,
    ystep=\cs
  ]
    (0,0) grid (\Vw,\Vh);

  \draw[
    vectorTeal,
    line width=1.2pt
  ]
    (0,0) rectangle (\Vw,\Vh);

\end{scope}

\node[
  text=vectorTeal,
  font=\bfseries
] at ($(Vc)+(0,-\Vh/2-0.5)$)
  {$\mathbf V_j$};

\node[
  text=vectorTeal,
  font=\normalsize,
  align=center
] at ($(Vc)+(0,-\Vh/2-1.0)$)
  {%
    vector stalk
    $\mathbf V_j\in\mathbb R^{n\times c_v}$
  };


\draw[
  coordPurple,
  line width=1.0pt,
  <->
]
  ($(Vsw)+(-0.24,0.05)$)
  --
  ($(Vsw)+(-0.24,\Vh)$);

\node[
  rotate=90,
  anchor=south,
  font=\scriptsize\bfseries,
  text=coordPurple
] at ($(Vsw)+(-0.28,\Vh-1.1)$)
  {ambient space $\mathbb R^n$ ($n$ rows)};

\draw[
  scalarSlate,
  line width=1.0pt,
  <->
]
  ($(Vsw)+(0.05,\Vh+0.45)$)
  --
  ($(Vsw)+(\Vw-0.05,\Vh+0.45)$);

\node[
  anchor=south,
  font=\scriptsize\bfseries,
  text=scalarSlate
] at ($(Vsw)+(\Vw/2,\Vh+0.55)$)
  {vector channels $\mathbb R^{c_v}$ ($c_v$ columns)};


\coordinate (Sc) at ($(Vc)+(-\Vw/2-4.05,0)$);

\node[
  block,
  draw=coordPurple,
  fill=coordPurple!8,
  text=coordPurple,
  minimum width=3.0cm,
  minimum height=2.3cm,
  line width=1.1pt
] (Sbox) at (Sc) {};

\node[
  text=coordPurple,
  font=\large\bfseries
] at ($(Sc)+(0,0.74)$)
  {$\mathbf S^{(k)}_{ij}
    \in\mathbb R^{n\times n}$};

\begin{scope}[shift={($(Sc)+(-0.40,-0.42)$)}]

  \draw[
    ->,
    line width=1.1pt,
    draw=coordPurple!30
  ]
    (0,0) -- (0.62,0);

  \draw[
    ->,
    line width=1.3pt,
    draw=coordPurple
  ]
    (0,0) -- (52:0.62);

  \draw[
    ->,
    line width=0.8pt,
    draw=coordPurple
  ]
    (0:0.74) arc (0:52:0.74);

  \fill[coordPurple]
    (0,0) circle (0.025);

\end{scope}

\node[
  font=\tiny,
  text=coordPurple
] at ($(Sc)+(0,-0.95)$)
  {Acts on ambient space};

\node[
  font=\scriptsize\bfseries\itshape,
  text=coordPurple,
  align=center
] at ($(Sc)+(0,1.78)$)
  {Left multiplication\\[-1pt]
   (acts on spatial rows)};

\coordinate (SarrA) at ($(Sbox.east)$);
\coordinate (SarrB) at ($(Vsw)+(-0.8,\Vh-1.1)$);

\draw[
  ->,
  coordPurple,
  line width=1.4pt
]
  (SarrA) -- (SarrB);


\coordinate (Mc) at ($(Vc)+(\Vw/2+4.05,0)$);

\node[
  block,
  draw=scalarSlate,
  fill=scalarSlate!8,
  text=scalarSlate,
  minimum width=3.0cm,
  minimum height=2.3cm,
  line width=1.1pt
] (Mbox) at (Mc) {};

\node[
  text=scalarSlate,
  font=\large\bfseries
] at ($(Mc)+(0,0.74)$)
  {$\mathbf M^{(k)}_{ij}
    \in\mathbb R^{c_v\times c_v}$};

\begin{scope}[shift={($(Mc)+(-0.86,-0.62)$)}]
  \bilinearChanBars{scalarSlate!40}
\end{scope}

\draw[
  scalarSlate,
  ->,
  line width=0.8pt
]
  ($(Mc)+(-0.10,-0.54)$)
  --
  ($(Mc)+(0.35,-0.20)$);

\draw[
  scalarSlate,
  ->,
  line width=0.8pt
]
  ($(Mc)+(-0.10,-0.30)$)
  --
  ($(Mc)+(0.35,-0.54)$);

\begin{scope}[shift={($(Mc)+(0.34,-0.62)$)}]

  \foreach \i/\h in {
    0/0.46,
    1/0.22,
    2/0.40,
    3/0.30,
    4/0.50
  }{%
    \fill[scalarSlate]
      (\i*0.16,0)
      rectangle
      (\i*0.16+0.11,\h);
  }

\end{scope}

\node[
  font=\tiny,
  text=scalarSlate
] at ($(Mc)+(0,-0.92)$)
  {Mixes vector channels};

\node[
  font=\scriptsize\bfseries\itshape,
  text=scalarSlate,
  align=center
] at ($(Mc)+(0,1.78)$)
  {Right multiplication\\[-1pt]
   (acts on channel columns)};

\coordinate (MarrA) at ($(Mbox.west)+(0,0.0)$);
\coordinate (MarrB) at ($(Vsw)+(\Vw+0.35,\Vh*0.50)$);

\draw[
  ->,
  scalarSlate,
  line width=1.4pt
]
  (MarrA) -- (MarrB);


\path[
  use as bounding box
]
  (-9.4,-2.8)
  rectangle
  (9.4,\Vh/2+4.85);

\end{tikzpicture}}
  \caption{
    \emph{Spatial--channel factorization of ESNN transport.}
    Each term acts on
    $\mathbf{V}_j\in\mathbb{R}^{n\times c_v}$ along two independent axes:
    $\mathbf{S}_{ij}^{(k)}$ acts on the ambient spatial dimension by left
    multiplication, while $\mathbf{M}_{ij}^{(k)}$ acts on the vector channels
    by right multiplication. Summing the separable terms gives the canonical
    transport $\mathcal{T}_{i\leftarrow j}$. ESNN enforces equivariance by
    constraining the spatial operators to be $O(n)$-covariant and the channel
    maps to be invariant.}
  \label{fig:bilinear_transport}
\end{figure}
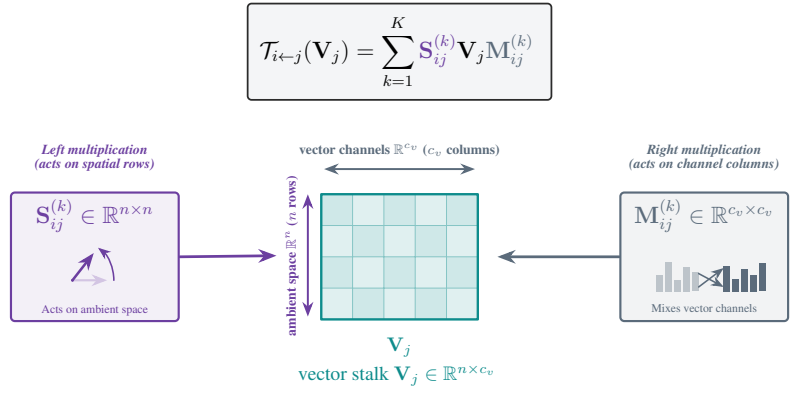

\subsection{Completeness Under Displacement-Only Conditioning}
\label{app:proof_thm1}

This subsection establishes the completeness result stated in
Theorem~\ref{thm:maximal_expressivity}. The key assumption is that relative
displacement is the only covariant geometric quantity available to the
transport, while any additional conditioning is $O(n)$-invariant. Under this
restriction, equivariance with respect to the stabilizer of a non-zero
displacement forces the spatial action to separate into radial and tangential
components. The proof below makes this constraint explicit and characterizes
the resulting transport class.

\begin{restatedtheorem}
{thm:maximal_expressivity}
\enspace
Let $n\geq2$, and consider a linear transport on
$\mathbb{R}^{n}\otimes\mathbb{R}^{c_v}$ whose only covariant geometric
conditioning variable is a non-zero displacement
$\mathbf{r}\in\mathbb{R}^{n}$, with arbitrary additional
$O(n)$-invariant scalar conditioning. Every such $O(n)$-equivariant transport
can be written as:
\begin{equation}
    \mathcal{T}_{\mathbf{r},\boldsymbol{\xi}}(\mathbf{V})
    =
    \mathbf{P}^{\parallel}(\mathbf{r})
    \mathbf{V}\mathbf{A}_{\mathbf{r},\boldsymbol{\xi}}
    +
    \mathbf{P}^{\perp}(\mathbf{r})
    \mathbf{V}\mathbf{B}_{\mathbf{r},\boldsymbol{\xi}}
\end{equation}
where
$\mathbf{A}_{\mathbf{r},\boldsymbol{\xi}}$ and
$\mathbf{B}_{\mathbf{r},\boldsymbol{\xi}}$
are channel endomorphisms whose dependence on the geometric conditioning is
$O(n)$-invariant.
\end{restatedtheorem}

\begin{proof}
Let
\begin{equation}
    \widehat{\mathbf{r}}
    =
    \frac{\mathbf{r}}{\|\mathbf{r}\|}
    \qquad
    \mathbf{P}^{\parallel}
    =
    \widehat{\mathbf{r}}\widehat{\mathbf{r}}^\top
    \qquad
    \mathbf{P}^{\perp}
    =
    \mathbf{I}_n-\mathbf{P}^{\parallel}
\end{equation}

For a fixed non-zero displacement $\mathbf{r}$, consider its stabilizer:
\begin{equation}
    H_{\mathbf{r}}
    =
    \left\{
        Q\in O(n):
        Q\mathbf{r}=\mathbf{r}
    \right\}
    \cong O(n-1)
\end{equation}
Under this subgroup, the spatial representation decomposes as:
\begin{equation}
    \mathbb{R}^{n}
    =
    \operatorname{span}\{\widehat{\mathbf{r}}\}
    \oplus
    \widehat{\mathbf{r}}^{\perp}
    \label{eq:stabilizer_decomposition}
\end{equation}
The stabilizer acts trivially on the radial component
$\operatorname{span}\{\widehat{\mathbf{r}}\}$ and through the standard
$O(n-1)$ representation on the tangential component
$\widehat{\mathbf{r}}^{\perp}$.

Let $\boldsymbol{\xi}$ denote any additional $O(n)$-invariant scalar
conditioning available to the transport, and consider:
\[
    \mathcal{T}_{\mathbf{r},\boldsymbol{\xi}}:
    \mathbb{R}^{n}\otimes\mathbb{R}^{c_v}
    \longrightarrow
    \mathbb{R}^{n}\otimes\mathbb{R}^{c_v}
\]
Since $\boldsymbol{\xi}$ is invariant, $O(n)$-equivariance requires:
\begin{equation}
    \mathcal{T}_{Q\mathbf{r},\boldsymbol{\xi}}
    (Q\mathbf{V})
    =
    Q\,
    \mathcal{T}_{\mathbf{r},\boldsymbol{\xi}}(\mathbf{V}),
    \qquad
    \forall Q\in O(n)
    \label{eq:displacement_transport_equivariance}
\end{equation}

For every $Q\in H_{\mathbf{r}}$, the displacement is unchanged, so
Equation~\ref{eq:displacement_transport_equivariance} reduces to:
\begin{equation}
    \mathcal{T}_{\mathbf{r},\boldsymbol{\xi}}(Q\mathbf{V})
    =
    Q\,
    \mathcal{T}_{\mathbf{r},\boldsymbol{\xi}}(\mathbf{V})
    \label{eq:stabilizer_commutation}
\end{equation}
Thus, for fixed $(\mathbf{r},\boldsymbol{\xi})$, the transport must commute
with the action of the stabilizer $H_{\mathbf{r}}$.

The two spatial subspaces in
Equation~\ref{eq:stabilizer_decomposition} carry inequivalent
representations of $H_{\mathbf{r}}$. The radial line carries the trivial
representation, whereas $\widehat{\mathbf{r}}^{\perp}$ carries the standard
representation of $O(n-1)$. An intertwining operator therefore cannot mix
the radial and tangential components. On the tangential subspace, the
spatial part of the commutant is proportional to the identity. Since
$\mathbb{R}^{c_v}$ carries no non-trivial $O(n)$ action, arbitrary linear
maps remain available on the channel dimension. Consequently:
\begin{equation}
    \operatorname{End}_{H_{\mathbf{r}}}
    \left(
        \mathbb{R}^{n}\otimes\mathbb{R}^{c_v}
    \right)
    =
    \left(
        \mathbf{P}^{\parallel}
        \otimes
        \operatorname{End}(\mathbb{R}^{c_v})
    \right)
    \oplus
    \left(
        \mathbf{P}^{\perp}
        \otimes
        \operatorname{End}(\mathbb{R}^{c_v})
    \right)
    \label{eq:stabilizer_commutant}
\end{equation}
Hence the transport must take the form:
\begin{equation}
    \boxed{
    \mathcal{T}_{\mathbf{r},\boldsymbol{\xi}}(\mathbf{V})
    =
    \mathbf{P}^{\parallel}(\mathbf{r})
    \mathbf{V}\mathbf{A}_{\mathbf{r},\boldsymbol{\xi}}
    +
    \mathbf{P}^{\perp}(\mathbf{r})
    \mathbf{V}\mathbf{B}_{\mathbf{r},\boldsymbol{\xi}}
    }
    \label{eq:appendix_radial_tangential_complete}
\end{equation}
for channel endomorphisms
$\mathbf{A}_{\mathbf{r},\boldsymbol{\xi}}$ and
$\mathbf{B}_{\mathbf{r},\boldsymbol{\xi}}$
in $\mathbb{R}^{c_v\times c_v}$.

It remains to determine how these channel maps may depend on the orientation
of $\mathbf{r}$. From
Equation~\ref{eq:displacement_transport_equivariance} and
\[
    \mathbf{P}^{\parallel}(Q\mathbf{r})
    =
    Q\mathbf{P}^{\parallel}(\mathbf{r})Q^\top,
    \qquad
    \mathbf{P}^{\perp}(Q\mathbf{r})
    =
    Q\mathbf{P}^{\perp}(\mathbf{r})Q^\top
\]
the uniqueness of the radial--tangential block decomposition gives:
\begin{equation}
    \mathbf{A}_{Q\mathbf{r},\boldsymbol{\xi}}
    =
    \mathbf{A}_{\mathbf{r},\boldsymbol{\xi}},
    \qquad
    \mathbf{B}_{Q\mathbf{r},\boldsymbol{\xi}}
    =
    \mathbf{B}_{\mathbf{r},\boldsymbol{\xi}}
\end{equation}
The channel endomorphisms can therefore depend on the displacement only
through $O(n)$-invariant quantities such as $\|\mathbf{r}\|$, together with
the additional invariant conditioning $\boldsymbol{\xi}$. This establishes
Equation~\ref{eq:appendix_radial_tangential_complete} as the complete class
of linear $O(n)$-equivariant transports under the stated conditioning
assumptions.
\end{proof}

Additional $O(n)$-invariant scalar information does not alter the
radial--tangential structure. It may change the channel maps
$\mathbf{A}_{\mathbf{r},\boldsymbol{\xi}}$ and
$\mathbf{B}_{\mathbf{r},\boldsymbol{\xi}}$, but cannot introduce additional
spatial operators. ESNN instantiates these channel maps through the structured factorization $\mathbf{W}\mathbf{D}(\boldsymbol{\eta}_{ij})$.

The completeness statement relies on the full orthogonal group $O(n)$.
Under $SO(n)$ alone, additional orientation-sensitive spatial operators may
be admissible. In three dimensions, for example, consider
$[\widehat{\mathbf{r}}]_{\times}$ defined by:
\begin{equation}
    [\widehat{\mathbf{r}}]_{\times}\mathbf{v}
    =
    \widehat{\mathbf{r}}\times\mathbf{v}
\end{equation}
For $Q\in O(3)$, this operator transforms as:
\begin{equation}
    [Q\widehat{\mathbf{r}}]_{\times}
    =
    \det(Q)\,
    Q[\widehat{\mathbf{r}}]_{\times}Q^\top
\end{equation}
It therefore satisfies the required conjugation law for $Q\in SO(3)$ but
acquires an additional sign under reflections.
Theorem~\ref{thm:maximal_expressivity} should consequently be understood as
a completeness result for displacement-conditioned linear
$O(n)$-equivariant transport, not for the broader class of
$SO(n)$-equivariant kernels.

\subsection{Feature-Conditioned Spatial Operators}
\label{app:feature_conditioned_transport}

The completeness result of Appendix~\ref{app:proof_thm1} applies when
relative displacement is the only covariant geometric input available to the
transport. ESNN can also construct spatial operators from the learned vector
features at the endpoints of an edge. These additional covariant features
provide a richer local geometric description and support transport mechanisms
beyond the radial--tangential class. The resulting spatial operators may vary
with the evolving feature state while retaining the spatial--channel structure
introduced in Section~\ref{sec:canonical_form}.

A simple example is the cross-feature operator:
\begin{equation}
    \mathbf{C}_{ij}
    =
    \mathbf{V}_j\mathbf{V}_i^\top
\end{equation}
which transforms by conjugation under a common orthogonal transformation:
\begin{equation}
    \mathbf{C}_{ij}
    \mapsto
    Q\mathbf{C}_{ij}Q^\top
\end{equation}
Its skew-symmetric component:
\begin{equation}
    \mathbf{\Omega}^{V}_{ij}
    =
    \mathbf{V}_j\mathbf{V}_i^\top
    -
    \mathbf{V}_i\mathbf{V}_j^\top
\end{equation}
inherits the same transformation law:
\begin{equation}
    \mathbf{\Omega}^{V}_{ij}
    \mapsto
    Q\mathbf{\Omega}^{V}_{ij}Q^\top
\end{equation}
Since the Frobenius norm is invariant under orthogonal conjugation, the
normalized operator:
\begin{equation}
    \widehat{\mathbf{\Omega}}^{V}_{ij}
    =
    \frac{
        \mathbf{\Omega}^{V}_{ij}
    }{
        \|\mathbf{\Omega}^{V}_{ij}\|_F+\varepsilon
    }
\end{equation}
is also a valid $O(n)$-covariant spatial operator. This gives a
feature-conditioned skew component that can be combined with the
displacement-based operators introduced in Section~\ref{sec:transports}.

Orthogonal Transport follows the same principle, with normalization applied
before extracting the skew component. Starting from
$\mathbf{C}_{ij}=\mathbf{V}_j\mathbf{V}_i^\top$, define:
\begin{equation}
    \widetilde{\mathbf{C}}_{ij}
    =
    \frac{
        \mathbf{C}_{ij}
    }{
        \|\mathbf{C}_{ij}\|_F+\varepsilon
    },
    \qquad
    \mathbf{\Omega}_{ij}
    =
    \widetilde{\mathbf{C}}_{ij}
    -
    \widetilde{\mathbf{C}}_{ij}^{\top}
\end{equation}
The invariance of the Frobenius norm gives:
\begin{equation}
    \mathbf{\Omega}_{ij}
    \mapsto
    Q\mathbf{\Omega}_{ij}Q^\top
\end{equation}
and for an invariant scalar $\beta_{ij}$, the matrix exponential preserves
the same conjugation law:
\begin{equation}
    \mathbf{R}_{ij}
    =
    \exp\!\left(
        \beta_{ij}\mathbf{\Omega}_{ij}
    \right)
    \mapsto
    Q\mathbf{R}_{ij}Q^\top
\end{equation}
Since $\mathbf{\Omega}_{ij}$ is skew-symmetric,
$\mathbf{R}_{ij}\in SO(n)$, recovering the spatial factor used by
Orthogonal Transport in Section~\ref{sec:transport_II}.

These constructions explain why the displacement-only completeness result
does not extend directly to feature-conditioned transport. Additional
covariant vector features provide geometric information beyond the single
direction $\mathbf{r}_{ij}$, so the stabilizer argument restricting the
spatial action to $\mathbf{P}^{\parallel}$ and
$\mathbf{P}^{\perp}$ no longer applies in general.

\paragraph{Unified Transport.}
\label{app:transport_unified}

Displacement- and feature-conditioned operators can be combined within the
general transport form of Equation~\ref{eq:canonical_transport}. Let
\[
    \mathcal{K}
    \subseteq
    \{
        \mathrm{id},
        \parallel,
        \perp,
        \mathrm{skew}
    \}
\]
denote the active spatial components, with
\begin{equation}
    \mathbf{B}^{(\mathrm{id})}_{ij}
    =
    \mathbf{I}_n
    \qquad
    \mathbf{B}^{(\parallel)}_{ij}
    =
    \mathbf{P}^{\parallel}_{ij}
    \qquad
    \mathbf{B}^{(\perp)}_{ij}
    =
    \mathbf{P}^{\perp}_{ij}
    \qquad
    \mathbf{B}^{(\mathrm{skew})}_{ij}
    =
    \widehat{\mathbf{\Omega}}^{V}_{ij}
\end{equation}
The corresponding transport is
\begin{equation}
    \mathcal{T}_{i\leftarrow j}(\mathbf{V}_j)
    =
    \sum_{b\in\mathcal{K}}
    \mathbf{B}^{(b)}_{ij}
    \mathbf{V}_j
    \mathbf{M}^{(b)}_{ij}
    \label{eq:unified_transport}
\end{equation}
Each spatial component satisfies the covariance condition of
Equation~\ref{eq:covariance_condition}, so the resulting map is
$O(n)$-equivariant whenever the corresponding channel maps are invariant.

This unified construction shows how learned covariant features can extend the
spatial operator family beyond the displacement-only class characterized by
Theorem~\ref{thm:maximal_expressivity}. It provides a convenient way to
combine several admissible spatial actions within a single transport map,
without implying that this particular set of operators spans the complete
class of feature-conditioned $O(n)$-equivariant transports.

\section{Cellular Sheaves, Directed Transport, and Connections}
\label{app:sheaf_transport}

ESNN directly parameterizes transport between neighboring vector features.
This appendix places these maps within the sheaf-theoretic perspective
introduced in Section~\ref{sec:background}. We first relate classical
cellular-sheaf restriction maps to the effective couplings they induce between
node stalks, then show how parameterizing these couplings directly leads to a
directed quiver formulation. We subsequently identify the additional
conditions required for self-adjoint transport and for an exact
connection-sheaf realization. We also distinguish the common ambient
$O(n)$ action enforced by ESNN from independent changes of local reference
frame.

\subsection{Cellular Sheaves and Induced Node-to-Node Coupling}
\label{app:sheaf_operators}

The connection between cellular sheaves and ESNN is most easily seen through
the node-to-node interaction induced by the sheaf Laplacian. Let
$G=(\mathcal{V},\mathcal{E})$ be a graph. A cellular sheaf $\mathcal{F}$
assigns a vector space $\mathcal{F}(i)$ to each vertex
$i\in\mathcal{V}$ and a vector space $\mathcal{F}(e)$ to each edge
$e\in\mathcal{E}$. For every vertex--edge incidence
$i\trianglelefteq e$, these local spaces are related by a linear restriction
map:
\begin{equation}
    \rho_{i\to e}:
    \mathcal{F}(i)
    \longrightarrow
    \mathcal{F}(e)
\end{equation}
The restriction maps place the features associated with the endpoints of an
edge in a common edge space where their compatibility can be compared. After
choosing an orientation for each edge $e=\{i,j\}$, the degree-$0$ coboundary
measures this disagreement as:
\begin{equation}
    (\delta_{\mathcal{F}}\mathbf{h})_e
    =
    \rho_{i\to e}\mathbf{h}_i
    -
    \rho_{j\to e}\mathbf{h}_j
    \label{eq:appendix_coboundary}
\end{equation}
up to the chosen orientation convention. With standard Euclidean inner
products, the corresponding cellular-sheaf Laplacian is
$\mathbf{L}_{\mathcal{F}}
=
\delta_{\mathcal{F}}^{*}
\delta_{\mathcal{F}}$
where $^{*}$ denotes the adjoint. For a single edge $e=\{i,j\}$, its
contribution to the Laplacian on the two endpoint stalks is:
\begin{equation}
    \mathbf{L}_{e}
    =
    \begin{pmatrix}
        \rho_{i\to e}^{*}\rho_{i\to e}
        &
        -\rho_{i\to e}^{*}\rho_{j\to e}
        \\[2mm]
        -\rho_{j\to e}^{*}\rho_{i\to e}
        &
        \rho_{j\to e}^{*}\rho_{j\to e}
    \end{pmatrix}
    \label{eq:edge_sheaf_laplacian}
\end{equation}

The off-diagonal blocks reveal the effective interaction between neighboring
node stalks. In particular, information from node $j$ contributes to node $i$
through the composition:
\begin{equation}
    \rho_{i\to e}^{*}\rho_{j\to e}
    :
    \mathcal{F}(j)
    \longrightarrow
    \mathcal{F}(i)
    \label{eq:induced_sheaf_transport}
\end{equation}
up to the conventional minus sign in the Laplacian. The restriction at $j$
maps the feature into the common edge space while the adjoint of the
restriction at $i$ maps the resulting representation back to the stalk at
$i$. Their composition therefore acts as an induced linear transport between
the neighboring node spaces. This effective node-to-node coupling is the point of departure for ESNN.
Rather than parameterizing separate restriction maps through an explicit edge
stalk and obtaining their composition indirectly, ESNN parameterizes the
neighboring vector transport directly:
\begin{equation}
    \mathcal{T}_{i\leftarrow j}:
    \mathcal{F}_{\mathrm{vec}}(j)
    \longrightarrow
    \mathcal{F}_{\mathrm{vec}}(i)
\end{equation}
The transport is then constrained to satisfy the transformation law of the
ambient Euclidean representation as developed in
Sections~\ref{sec:canonical_form} and~\ref{sec:transports}. This retains the
sheaf perspective of local feature spaces connected by linear maps while
allowing the two orientations of an interaction to be parameterized directly.
In general, however, the resulting ESNN operator need not arise from a
classical cellular-sheaf Laplacian. The conditions under which more structured
sheaf realizations are recovered are developed below.

\subsection{Directed Transport as a Quiver Representation}
\label{app:directed_transport}

Directly parameterizing node-to-node transport naturally accommodates
asymmetric interactions. A convenient language for describing this structure
is that of quiver representations. A quiver is a directed graph whose
vertices are assigned vector spaces and whose arrows are assigned linear maps
between those spaces. In ESNN, the vertices correspond to node vector stalks
and the arrows to the directed transport maps used during message passing.

For a given layer state, once the context-dependent transport operators have
been determined, let
\begin{equation}
    \mathcal{Q}
    =
    (\mathcal{V},\mathcal{A})
\end{equation}
denote the directed interaction graph with an arrow $j\to i$ for every
directed message-passing interaction. Assign the vector space:
\begin{equation}
    \mathcal{F}_{\mathrm{vec}}(i)
    =
    \mathbb{R}^{n}\otimes\mathbb{R}^{c_v}
\end{equation}
to each vertex $i$ and the realized linear transport:
\begin{equation}
    \mathcal{T}_{i\leftarrow j}:
    \mathcal{F}_{\mathrm{vec}}(j)
    \longrightarrow
    \mathcal{F}_{\mathrm{vec}}(i)
\end{equation}
to each arrow $j\to i$. These assignments define a representation of the
quiver $\mathcal{Q}$ for the realized transport operators at that layer.

Equivalently, the construction extends to the free path category of
$\mathcal{Q}$. Each directed path
\begin{equation}
    i_0
    \longrightarrow
    i_1
    \longrightarrow
    \cdots
    \longrightarrow
    i_m
\end{equation}
is associated with the composition
\begin{equation}
    \mathcal{T}_{i_m\leftarrow i_{m-1}}
    \circ
    \cdots
    \circ
    \mathcal{T}_{i_1\leftarrow i_0}
\end{equation}
so the vertex spaces and directed transports extend naturally to a covariant
functor from the free path category of $\mathcal{Q}$ to finite-dimensional
vector spaces. This categorical viewpoint is not required to define an ESNN
layer but makes explicit that directed transports compose consistently along
paths.

The ambient equivariance of the individual edge maps is preserved under this
composition. If every arrow satisfies:
\begin{equation}
    \mathcal{T}'_{i\leftarrow j}\circ Q
    =
    Q\circ\mathcal{T}_{i\leftarrow j}
\end{equation}
where $Q$ denotes the natural $O(n)$ action on
$\mathcal{F}_{\mathrm{vec}}
=
\mathbb{R}^{n}\otimes\mathbb{R}^{c_v}$, then for any path
$i_0\to i_1\to\cdots\to i_m$
\begin{equation}
\begin{aligned}
    \mathcal{T}'_{i_m\leftarrow i_{m-1}}
    \circ\cdots\circ
    \mathcal{T}'_{i_1\leftarrow i_0}
    \circ Q
    &=
    Q\circ
    \mathcal{T}_{i_m\leftarrow i_{m-1}}
    \circ\cdots\circ
    \mathcal{T}_{i_1\leftarrow i_0}
\end{aligned}
\end{equation}
Thus the transport associated with a path obeys the same global $O(n)$
intertwining relation as the individual edge maps. This describes the
algebraic composition of the transports while a single ESNN layer still
performs the one-hop aggregation defined in Section~\ref{sec:layer}.

The quiver formulation does not require any compatibility condition between
opposite edge orientations. If both $i\to j$ and $j\to i$ are present, the
general ESNN construction does not impose:
\begin{equation}
    \mathcal{T}_{i\leftarrow j}
    =
    \mathcal{T}_{j\leftarrow i}^{*},
    \qquad
    \mathcal{T}_{i\leftarrow j}
    =
    \mathcal{T}_{j\leftarrow i}^{-1}
\end{equation}
The two directions may therefore represent genuinely different local
interactions. This distinguishes the general directional ESNN operator from
diffusion generated by a classical self-adjoint sheaf Laplacian and is
consistent with recent directed extensions of sheaf-based learning
~\citep{ribeiro2025cooperative,fiorini2025sheaves}.

This perspective is also closely related to copresheaf neural constructions
which organize information through directed maps between local feature spaces
and their compositions~\citep{hajij2025copresheaf}. ESNN shares this
directed and compositional view of information flow while imposing an
additional geometric requirement: each transport must intertwine the common
ambient $O(n)$ action carried by the vector features. The resulting structure
is therefore directional at the level of the interaction graph and
equivariant with respect to the ambient geometry.

The transport operators are state dependent, with coefficients determined from
the current node, edge, and geometric features as described in
Section~\ref{sec:canonical_form}. Once these coefficients are fixed, each
realized edge operator retains the linear spatial--channel form introduced
above, even though the overall message function may depend nonlinearly on the
evolving feature state.

\subsection{Adjoint-Consistent and Connection-Sheaf Specializations}
\label{app:connection_specialization}

The directed formulation of Appendix~\ref{app:directed_transport} allows the
two orientations of an edge to carry independent transport maps. Classical
sheaf diffusion is more structured: the off-diagonal blocks of a
cellular-sheaf Laplacian occur in adjoint pairs which makes the resulting
operator self-adjoint. We first identify the corresponding specialization of
ESNN and then give the stronger conditions under which the edge transport
admits an exact connection-sheaf realization.

Recall the normalized transport operator:
\begin{equation}
    (\mathcal{A}_{\mathcal{T}}\mathbf{V})_i
    =
    \nu_{ii}\mathbf{V}_i
    +
    \sum_{j\in\mathcal{N}(i)}
    \nu_{ij}
    \mathcal{T}_{i\leftarrow j}(\mathbf{V}_j)
    \label{eq:appendix_transport_operator}
\end{equation}
For the left--right transport form introduced in
Equation~\ref{eq:canonical_transport}
\begin{equation}
    \mathcal{T}(\mathbf{V})
    =
    \sum_{k}
    \mathbf{S}^{(k)}
    \mathbf{V}
    \mathbf{M}^{(k)}
\end{equation}
the adjoint with respect to the Frobenius inner product is:
\begin{equation}
    \mathcal{T}^{*}(\mathbf{U})
    =
    \sum_{k}
    \left(\mathbf{S}^{(k)}\right)^{\top}
    \mathbf{U}
    \left(\mathbf{M}^{(k)}\right)^{\top}
    \label{eq:transport_adjoint}
\end{equation}
Indeed
\begin{equation}
    \left\langle
        \mathbf{U}
        \mathbf{S}\mathbf{V}\mathbf{M}
    \right\rangle_F
    =
    \left\langle
        \mathbf{S}^{\top}\mathbf{U}\mathbf{M}^{\top}
        \mathbf{V}
    \right\rangle_F
\end{equation}
For Radial--Tangential Transport, the spatial projectors are symmetric so the
adjoint acts by transposing the corresponding channel maps:
\begin{equation}
    \mathcal{T}^{*}(\mathbf{U})
    =
    \mathbf{P}^{\parallel}\mathbf{U}
    \left(\mathbf{M}^{\parallel}\right)^{\top}
    +
    \mathbf{P}^{\perp}\mathbf{U}
    \left(\mathbf{M}^{\perp}\right)^{\top}
\end{equation}

This makes the condition for self-adjoint transport explicit. On a bidirected
graph, the scalar normalization must agree across the two orientations and the
transport in one direction must be the adjoint of the transport in the other.

\begin{proposition}[Self-Adjoint Transport Operator]
\label{prop:self_adjoint_transport}
Fix the layer state so that the transport maps and aggregation weights are
determined. Suppose the graph is bidirected and that for every adjacent pair
$i,j$
\begin{equation}
    \nu_{ij}
    =
    \nu_{ji}
    \qquad
    \mathcal{T}_{j\leftarrow i}
    =
    \mathcal{T}_{i\leftarrow j}^{*}
    \label{eq:adjoint_consistency}
\end{equation}
Then $\mathcal{A}_{\mathcal{T}}$ is self-adjoint with respect to the
direct-sum Frobenius inner product on the node vector features.
\end{proposition}

\begin{proof}
For two node signals $\mathbf{U}$ and $\mathbf{V}$, an off-diagonal
contribution to
$\langle
\mathbf{U}
\mathcal{A}_{\mathcal{T}}\mathbf{V}
\rangle$
is
\begin{equation}
    \nu_{ij}
    \left\langle
        \mathbf{U}_i
        \mathcal{T}_{i\leftarrow j}\mathbf{V}_j
    \right\rangle_F
\end{equation}
Using the adjoint relation together with
$\nu_{ij}=\nu_{ji}$
\begin{equation}
    \nu_{ij}
    \left\langle
        \mathbf{U}_i
        \mathcal{T}_{i\leftarrow j}\mathbf{V}_j
    \right\rangle_F
    =
    \nu_{ji}
    \left\langle
        \mathcal{T}_{j\leftarrow i}\mathbf{U}_i
        \mathbf{V}_j
    \right\rangle_F
\end{equation}
Summing over the bidirected edge set gives
\begin{equation}
    \left\langle
        \mathbf{U}
        \mathcal{A}_{\mathcal{T}}\mathbf{V}
    \right\rangle
    =
    \left\langle
        \mathcal{A}_{\mathcal{T}}\mathbf{U}
        \mathbf{V}
    \right\rangle
\end{equation}
The diagonal self-loop terms are real scalar multiples of the identity and
are therefore self-adjoint.
\end{proof}

Proposition~\ref{prop:self_adjoint_transport} identifies the
adjoint-consistent specialization of the general directional operator. This
recovers the self-adjoint structure characteristic of classical sheaf
diffusion, but does not by itself make
$\mathcal{A}_{\mathcal{T}}$ a cellular-sheaf Laplacian. An exact
connection-sheaf realization additionally requires the transport between
neighboring fibers to be orthogonal, with the reverse direction given by its
inverse and hence by its adjoint. Under these stronger conditions, the ESNN
edge transport can be represented exactly through classical sheaf restriction
maps.

\begin{proposition}[Connection-Sheaf Realization]
\label{prop:connection_sheaf_realization}
Let
$\mathcal{H}
=
\mathbb{R}^{n}\otimes\mathbb{R}^{c_v}$
be the vector fiber and consider an undirected edge $e=\{i,j\}$. Suppose the
full transport
\begin{equation}
    \mathcal{U}_{i\leftarrow j}:
    \mathcal{H}
    \longrightarrow
    \mathcal{H}
\end{equation}
is orthogonal and that the reverse transport satisfies
\begin{equation}
    \mathcal{U}_{j\leftarrow i}
    =
    \mathcal{U}_{i\leftarrow j}^{*}
    =
    \mathcal{U}_{i\leftarrow j}^{-1}
    \label{eq:orthogonal_reverse_transport}
\end{equation}
Choose vertex and edge stalks equal to $\mathcal{H}$ and restrictions
\begin{equation}
    \rho_{i\to e}
    =
    \mathbf{I}_{\mathcal{H}}
    \qquad
    \rho_{j\to e}
    =
    \mathcal{U}_{i\leftarrow j}
    \label{eq:connection_restrictions}
\end{equation}
Then the contribution of $e$ to the cellular-sheaf Laplacian is
\begin{equation}
    \mathbf{L}_e
    =
    \begin{pmatrix}
        \mathbf{I}_{\mathcal{H}}
        &
        -\mathcal{U}_{i\leftarrow j}
        \\[1mm]
        -\mathcal{U}_{i\leftarrow j}^{*}
        &
        \mathbf{I}_{\mathcal{H}}
    \end{pmatrix}
    \label{eq:connection_edge_laplacian}
\end{equation}
Thus the off-diagonal node couplings are exactly the orthogonal transport and
its adjoint up to the conventional Laplacian sign.
\end{proposition}

\begin{proof}
Substituting Equation~\ref{eq:connection_restrictions} into
Equation~\ref{eq:edge_sheaf_laplacian} gives
\begin{equation}
    \rho_{i\to e}^{*}\rho_{i\to e}
    =
    \mathbf{I}_{\mathcal{H}}
    \qquad
    \rho_{i\to e}^{*}\rho_{j\to e}
    =
    \mathcal{U}_{i\leftarrow j}
\end{equation}
Since $\mathcal{U}_{i\leftarrow j}$ is orthogonal
\begin{equation}
    \rho_{j\to e}^{*}\rho_{j\to e}
    =
    \mathcal{U}_{i\leftarrow j}^{*}
    \mathcal{U}_{i\leftarrow j}
    =
    \mathbf{I}_{\mathcal{H}}
\end{equation}
while
\begin{equation}
    \rho_{j\to e}^{*}\rho_{i\to e}
    =
    \mathcal{U}_{i\leftarrow j}^{*}
\end{equation}
Equation~\ref{eq:connection_edge_laplacian} follows.
\end{proof}

Proposition~\ref{prop:connection_sheaf_realization} gives the discrete
orthogonal-connection counterpart of ESNN transport and is closely related to
connection-Laplacian SNNs~\citep{barbero2022sheaf}. In this sense,
Orthogonal Transport matches the required geometry through its spatial factor
$\mathbf{R}_{ij}\in SO(n)$. The full ESNN edge map, however, also includes
channel mixing and edge-dependent gating, and is therefore not generally
orthogonal on
$\mathbb{R}^{n}\otimes\mathbb{R}^{c_v}$. The exact connection-sheaf
realization is recovered when the complete fiber transport is orthogonal and
the reverse direction is given by its adjoint.

This yields a simple hierarchy. General ESNN transport allows independent
directed edge maps. Enforcing
Equation~\ref{eq:adjoint_consistency} gives a self-adjoint transport operator,
while additionally requiring orthogonality of the full fiber maps recovers the
classical connection-sheaf case of
Proposition~\ref{prop:connection_sheaf_realization}.

\subsection{Ambient Equivariance and Local Gauge Transformations}
\label{app:ambient_vs_gauge}

The connection interpretation above should not be confused with local gauge
equivariance. In a classical cellular sheaf, the bases of different stalks
may be changed independently. If $\mathbf{Q}_i$ and $\mathbf{Q}_e$ are
orthogonal changes of basis on a vertex stalk and an incident edge stalk, the
coordinate representation of the corresponding restriction map transforms as:
\begin{equation}
    \rho'_{i\to e}
    =
    \mathbf{Q}_e
    \rho_{i\to e}
    \mathbf{Q}_i^{\top}
    \label{eq:sheaf_basis_change}
\end{equation}
This expresses covariance under independent changes of local reference frame. ESNN enforces a different symmetry. A single orthogonal transformation
$Q\in O(n)$ acts simultaneously on the physical coordinates and all vector
features:
\begin{equation}
    \mathbf{x}_i
    \mapsto
    Q\mathbf{x}_i+\mathbf{t},
    \qquad
    \mathbf{V}_i
    \mapsto
    Q\mathbf{V}_i
\end{equation}
while the spatial transport transforms as:
\begin{equation}
    \mathbf{S}_{ij}
    \mapsto
    Q\mathbf{S}_{ij}Q^\top
\end{equation}
This is the ambient $E(n)$-equivariance established in
Sections~\ref{sec:canonical_form}--\ref{sec:layer}. It corresponds to
transforming the entire geometric system in a common Cartesian frame rather
than independently changing the basis attached to each node.

In particular, the feature-derived Orthogonal Transport is not designed to
be covariant under arbitrary node-wise frame changes. For
$\mathbf{C}_{ij}
=
\mathbf{V}_j\mathbf{V}_i^\top$
independent transformations:
\begin{equation}
    \mathbf{V}_i
    \mapsto
    \mathbf{Q}_i\mathbf{V}_i,
    \qquad
    \mathbf{V}_j
    \mapsto
    \mathbf{Q}_j\mathbf{V}_j
\end{equation}
give:
\begin{equation}
    \mathbf{C}'_{ij}
    =
    \mathbf{Q}_j
    \mathbf{C}_{ij}
    \mathbf{Q}_i^\top
\end{equation}
and the skew-symmetric construction used by Orthogonal Transport does not in
general transform as a gauge-covariant map between the two independently
chosen frames. When $\mathbf{Q}_i=\mathbf{Q}_j=Q$, however, the same
construction reduces to conjugation by the common ambient transformation
which is precisely the covariance required by ESNN. Thus, Orthogonal Transport has a connection-like geometric interpretation
under the common ambient frame but is not a general gauge-equivariant
transport between independently chosen local frames.

A gauge-equivariant ESNN would instead associate each node with its own local
frame and require directed transports to transform as:
\begin{equation}
    \mathcal{T}'_{i\leftarrow j}
    =
    \mathbf{Q}_i
    \mathcal{T}_{i\leftarrow j}
    \mathbf{Q}_j^\top
\end{equation}
under independent local transformations $\mathbf{Q}_i$ and $\mathbf{Q}_j$.
Such a construction would require transport operators built from
gauge-covariant quantities rather than from the common-frame feature
representations used here. Extending ESNN to this locally gauge-equivariant
setting is left for future work.


\section{Proofs of Equivariance and Symmetry Results}
\label{app:proofs}

This appendix provides the proofs of the equivariance and symmetry results
stated in the main text. The algebraic proof of the completeness of
displacement-conditioned linear transport is given separately in
Appendix~\ref{app:proof_thm1}, while the self-adjoint and connection-sheaf
specializations are established in Appendix~\ref{app:sheaf_transport}.

\subsection{Proof of Proposition~\ref{prop:transport_equivariance}:
\texorpdfstring{$O(n)$}{O(n)}-Equivariance of the Transport Map}
\label{app:proof_transport_equivariance}

\begin{restatedproposition}
[$O(n)$-Equivariance of the Transport Map]
{prop:transport_equivariance}
\enspace
Suppose that, for every component $k$, the spatial operator satisfies
\begin{equation}
    \mathbf{S}_{ij}^{(k)}(Q\!\cdot\!\mathcal{C}_{ij})
    =
    Q\mathbf{S}_{ij}^{(k)}(\mathcal{C}_{ij})Q^\top
\end{equation}
and the channel operator $\mathbf{M}_{ij}^{(k)}$ is unchanged under the
$O(n)$ action. Then the transport map in
Equation~\ref{eq:canonical_transport} is $O(n)$-equivariant
\begin{equation}
    \mathcal{T}'_{i\leftarrow j}(Q\mathbf{V}_j)
    =
    Q\,\mathcal{T}_{i\leftarrow j}(\mathbf{V}_j)
\end{equation}
where $\mathcal{T}'_{i\leftarrow j}$ denotes the transport evaluated from the
transformed edge context.
\end{restatedproposition}

\begin{proof}
Let $Q\in O(n)$ and consider the transport evaluated from the transformed
edge context $Q\!\cdot\!\mathcal{C}_{ij}$. By assumption
\begin{equation}
    \mathbf{S}_{ij}^{(k)}(Q\!\cdot\!\mathcal{C}_{ij})
    =
    Q\mathbf{S}_{ij}^{(k)}(\mathcal{C}_{ij})Q^\top
\end{equation}
while the channel operators satisfy
\begin{equation}
    \mathbf{M}_{ij}^{(k)}(Q\!\cdot\!\mathcal{C}_{ij})
    =
    \mathbf{M}_{ij}^{(k)}(\mathcal{C}_{ij})
\end{equation}
Evaluating the transport on the transformed vector feature gives
\begin{align}
    \mathcal{T}'_{i\leftarrow j}(Q\mathbf{V}_j)
    &=
    \sum_{k=1}^{K}
    \left(
        Q\mathbf{S}_{ij}^{(k)}Q^\top
    \right)
    (Q\mathbf{V}_j)
    \mathbf{M}_{ij}^{(k)}
    \\
    &=
    Q
    \sum_{k=1}^{K}
    \mathbf{S}_{ij}^{(k)}
    \mathbf{V}_j
    \mathbf{M}_{ij}^{(k)}
    \\
    &=
    Q\,\mathcal{T}_{i\leftarrow j}(\mathbf{V}_j)
\end{align}
which is the required $O(n)$-equivariance relation.
\end{proof}

\subsection{\texorpdfstring{$O(n)$}{O(n)}-Equivariance of Orthogonal Transport}
\label{app:proof_lemma1}

\begin{proposition}[Equivariance of Orthogonal Transport]
Let $Q\in O(n)$ act on the vector features as
$\mathbf{V}_i\mapsto Q\mathbf{V}_i$. Then the Orthogonal Transport of
Equation~\ref{eq:orthogonal_transport} satisfies
\begin{equation}
    \mathcal{T}'_{i\leftarrow j}(Q\mathbf{V}_j)
    =
    Q\,\mathcal{T}_{i\leftarrow j}(\mathbf{V}_j)
\end{equation}
where the primed transport is evaluated from the transformed edge context.
\end{proposition}

\begin{proof}
Recall that Orthogonal Transport can be written as
\begin{equation}
    \mathcal{T}_{i\leftarrow j}(\mathbf{V}_j)
    =
    \mathbf{R}_{ij}\mathbf{V}_j\mathbf{M}_{ij}
    \qquad
    \mathbf{M}_{ij}
    =
    \mathbf{W}\mathbf{D}(\boldsymbol{\eta}_{ij})
    \label{eq:proof_orthogonal_transport}
\end{equation}
with
\begin{equation}
    \mathbf{R}_{ij}
    =
    \exp\!\left(
        \beta_{ij}\mathbf{\Omega}_{ij}
    \right)
\end{equation}
and
\begin{equation}
    \mathbf{\Omega}_{ij}
    =
    \widetilde{\mathbf{C}}_{ij}
    -
    \widetilde{\mathbf{C}}_{ij}^{\top}
    \qquad
    \widetilde{\mathbf{C}}_{ij}
    =
    \frac{
        \mathbf{C}_{ij}
    }{
        \|\mathbf{C}_{ij}\|_F+\varepsilon
    }
    \qquad
    \mathbf{C}_{ij}
    =
    \mathbf{V}_j\mathbf{V}_i^\top
    \label{eq:proof_cross_feature}
\end{equation}

Since
$\mathbf{\Omega}_{ij}^{\top}=-\mathbf{\Omega}_{ij}$,
we have
$\beta_{ij}\mathbf{\Omega}_{ij}\in\mathfrak{so}(n)$ and therefore
$\mathbf{R}_{ij}\in SO(n)$.

Under the transformations
$\mathbf{V}'_i=Q\mathbf{V}_i$ and
$\mathbf{V}'_j=Q\mathbf{V}_j$
the cross-feature matrix becomes
\begin{equation}
    \mathbf{C}'_{ij}
    =
    (Q\mathbf{V}_j)
    (Q\mathbf{V}_i)^\top
    =
    Q\mathbf{C}_{ij}Q^\top
    \label{eq:proof_cross_feature_covariance}
\end{equation}

The Frobenius norm is invariant under orthogonal conjugation, hence
\begin{equation}
    \widetilde{\mathbf{C}}'_{ij}
    =
    Q\widetilde{\mathbf{C}}_{ij}Q^\top
\end{equation}
and consequently
\begin{equation}
    \mathbf{\Omega}'_{ij}
    =
    Q\mathbf{\Omega}_{ij}Q^\top
    \label{eq:proof_omega_covariance}
\end{equation}

The coefficient $\beta_{ij}$ is predicted from invariant edge features and
therefore $\beta'_{ij}=\beta_{ij}$. The matrix exponential is compatible
with similarity transformations
\begin{equation}
    \exp(Q\mathbf{A}Q^{-1})
    =
    Q\exp(\mathbf{A})Q^{-1}
    \label{eq:matrix_exp_similarity}
\end{equation}
and since $Q^{-1}=Q^\top$
\begin{equation}
    \mathbf{R}'_{ij}
    =
    Q\mathbf{R}_{ij}Q^\top
    \label{eq:proof_rotation_covariance}
\end{equation}

The gating vector $\boldsymbol{\eta}_{ij}$ is likewise constructed from
invariant quantities, so $\mathbf{M}'_{ij}=\mathbf{M}_{ij}$. Evaluating the transport from the
transformed context gives
\begin{align}
    \mathcal{T}'_{i\leftarrow j}(Q\mathbf{V}_j)
    &=
    (Q\mathbf{R}_{ij}Q^\top)
    (Q\mathbf{V}_j)
    \mathbf{M}_{ij}
    \\
    &=
    Q
    \mathbf{R}_{ij}
    \mathbf{V}_j
    \mathbf{M}_{ij}
    \\
    &=
    Q\,
    \mathcal{T}_{i\leftarrow j}(\mathbf{V}_j)
\end{align}
which proves the claim.
\end{proof}

\subsection{\texorpdfstring{$O(n)$}{O(n)}-Equivariance of Radial--Tangential Transport}
\label{app:proof_lemma2}

\begin{proposition}[Equivariance of Radial--Tangential Transport]
For $\mathbf{r}_{ij}\neq\mathbf{0}$, the Radial--Tangential Transport
\begin{equation}
    \mathcal{T}_{i\leftarrow j}(\mathbf{V}_j)
    =
    \mathbf{P}^{\parallel}_{ij}
    \mathbf{V}_j
    \mathbf{M}^{\parallel}_{ij}
    +
    \mathbf{P}^{\perp}_{ij}
    \mathbf{V}_j
    \mathbf{M}^{\perp}_{ij}
\end{equation}
is $O(n)$-equivariant.
\end{proposition}

\begin{proof}
Let $Q\in O(n)$. Since
$\mathbf{r}'_{ij}=Q\mathbf{r}_{ij}$ and $Q$ is orthogonal
\begin{equation}
    \|\mathbf{r}'_{ij}\|
    =
    \|\mathbf{r}_{ij}\|
\end{equation}
and therefore
\begin{equation}
    \widehat{\mathbf{r}}'_{ij}
    =
    Q\widehat{\mathbf{r}}_{ij}
\end{equation}

The radial projector transforms by conjugation
\begin{equation}
    (\mathbf{P}^{\parallel}_{ij})'
    =
    \widehat{\mathbf{r}}'_{ij}
    (\widehat{\mathbf{r}}'_{ij})^\top
    =
    Q\mathbf{P}^{\parallel}_{ij}Q^\top
    \label{eq:proof_parallel_projector}
\end{equation}
Since
$\mathbf{P}^{\perp}_{ij}
=
\mathbf{I}_n-\mathbf{P}^{\parallel}_{ij}$
we also obtain
\begin{equation}
    (\mathbf{P}^{\perp}_{ij})'
    =
    Q\mathbf{P}^{\perp}_{ij}Q^\top
    \label{eq:proof_perpendicular_projector}
\end{equation}

The channel maps
$\mathbf{M}^{\parallel}_{ij}$ and
$\mathbf{M}^{\perp}_{ij}$ are constructed from invariant edge features and
are therefore unchanged by the $O(n)$ action. Hence
\begin{align}
    \mathcal{T}'_{i\leftarrow j}(Q\mathbf{V}_j)
    &=
    (Q\mathbf{P}^{\parallel}_{ij}Q^\top)
    (Q\mathbf{V}_j)
    \mathbf{M}^{\parallel}_{ij}
    +
    (Q\mathbf{P}^{\perp}_{ij}Q^\top)
    (Q\mathbf{V}_j)
    \mathbf{M}^{\perp}_{ij}
    \\
    &=
    Q
    \left(
        \mathbf{P}^{\parallel}_{ij}
        \mathbf{V}_j
        \mathbf{M}^{\parallel}_{ij}
        +
        \mathbf{P}^{\perp}_{ij}
        \mathbf{V}_j
        \mathbf{M}^{\perp}_{ij}
    \right)
    \\
    &=
    Q\,
    \mathcal{T}_{i\leftarrow j}(\mathbf{V}_j)
\end{align}
which proves the claim.
\end{proof}

\subsection{Completeness of Displacement-Conditioned Linear Transport}

Theorem~\ref{thm:maximal_expressivity} is proved in
Appendix~\ref{app:proof_thm1}. The argument characterizes the commutant of
the stabilizer
\[
    H_{\mathbf{r}}
    =
    \{Q\in O(n):Q\mathbf{r}=\mathbf{r}\}
    \cong O(n-1)
\]
for a non-zero displacement. This yields the two spatial projectors
$\mathbf{P}^{\parallel}$ and $\mathbf{P}^{\perp}$ together with arbitrary
channel endomorphisms whose dependence on the geometric conditioning is
$O(n)$-invariant.

\subsection{Proof of Theorem~\ref{thm:layer_equivariance}:
\texorpdfstring{$E(n)$}{E(n)}-Equivariance of the ESNN Layer}
\label{app:proof_thm2}

\begin{restatedtheorem}
[$E(n)$-Equivariance]
{thm:layer_equivariance}
\enspace
Assume that the graph topology is fixed or constructed from
$E(n)$-invariant quantities, that edge attributes are $O(n)$-invariant, and
that the transport maps satisfy
Proposition~\ref{prop:transport_equivariance}. Assume also that
$\mathbf{r}_{ij}\neq\mathbf{0}$ whenever a normalized edge direction is used.
In the absence of the symmetry-relaxing inputs introduced in
Section~\ref{sec:soft_symmetry}, the ESNN layer of
Section~\ref{sec:layer} is $E(n)$-equivariant.
\end{restatedtheorem}

\begin{proof}
Let $(Q,\mathbf{t})\in E(n)$ with $Q\in O(n)$ and transform
\begin{equation}
    \overline{\mathbf{x}}_i
    =
    Q\mathbf{x}_i+\mathbf{t}
    \qquad
    \overline{\mathbf{V}}_i
    =
    Q\mathbf{V}_i
    \qquad
    \overline{\mathbf{s}}_i
    =
    \mathbf{s}_i
    \label{eq:proof_layer_action}
\end{equation}
For the first layer, the lift condition in
Equation~\ref{eq:vector_lift} ensures that any task-specific vector
initialization obeys the same transformation law. By assumption, the graph
topology and edge attributes are unchanged.

For each interaction
\begin{equation}
    \overline{\mathbf{r}}_{ij}
    =
    \overline{\mathbf{x}}_i-\overline{\mathbf{x}}_j
    =
    Q\mathbf{r}_{ij}
    \label{eq:proof_relative_vector}
\end{equation}
Orthogonality preserves distances, vector norms, and inner products. All
quantities entering the invariant construction of
Equation~\ref{eq:edge_context} are therefore unchanged
\begin{equation}
    \overline{\mathbf{z}}_{ij}
    =
    \mathbf{z}_{ij}
    \label{eq:proof_edge_context_invariant}
\end{equation}

By Proposition~\ref{prop:transport_equivariance}, the vector messages satisfy
\begin{equation}
    \overline{\mathbf{m}}^{V}_{i\leftarrow j}
    =
    Q\mathbf{m}^{V}_{i\leftarrow j}
    \label{eq:proof_vector_message}
\end{equation}
The scalar messages depend only on invariant quantities and therefore satisfy
\begin{equation}
    \overline{\mathbf{m}}^{s}_{i\leftarrow j}
    =
    \mathbf{m}^{s}_{i\leftarrow j}
    \label{eq:proof_scalar_message}
\end{equation}

The normalized coefficients $\nu_{ij}$ and $\nu_{ii}$ are invariant scalars.
Equation~\ref{eq:esnn_aggregation} therefore gives
\begin{equation}
    \overline{\mathbf{V}}^{\mathrm{agg}}_i
    =
    Q\mathbf{V}^{\mathrm{agg}}_i
    \qquad
    \overline{\mathbf{s}}^{\mathrm{agg}}_i
    =
    \mathbf{s}^{\mathrm{agg}}_i
    \label{eq:proof_aggregation}
\end{equation}

The matrix $\mathbf{W}_V$ acts only on vector channels and therefore commutes
with the spatial action
\begin{equation}
    (Q\mathbf{V}^{\mathrm{agg}}_i)\mathbf{W}_V
    =
    Q(\mathbf{V}^{\mathrm{agg}}_i\mathbf{W}_V)
\end{equation}
The radial non-linearity is equivariant because
\begin{equation}
    \sigma_V(Q\mathbf{v})
    =
    a(\|Q\mathbf{v}\|)Q\mathbf{v}
    =
    Q\sigma_V(\mathbf{v})
    \label{eq:proof_radial_nonlinearity}
\end{equation}
The scalar update acts only on invariant quantities through channel maps and
scalar nonlinearities and therefore remains invariant. Equation~\ref{eq:node_update}
consequently yields
\begin{equation}
    \overline{\mathbf{V}}'_i
    =
    Q\mathbf{V}'_i
    \qquad
    \overline{\mathbf{s}}'_i
    =
    \mathbf{s}'_i
    \label{eq:proof_feature_update}
\end{equation}

Finally, $\gamma_{ij}$ in Equation~\ref{eq:coord_update} is predicted from
invariant node and edge features and is therefore unchanged. Using
Equation~\ref{eq:proof_relative_vector}
\begin{equation}
    \overline{\Delta\mathbf{x}}_i
    =
    \frac{1}{|\mathcal{N}(i)|}
    \sum_{j\in\mathcal{N}(i)}
    \gamma_{ij}Q\mathbf{r}_{ij}
    =
    Q\Delta\mathbf{x}_i
    \label{eq:proof_coordinate_displacement}
\end{equation}
and hence
\begin{equation}
    \overline{\mathbf{x}}'_i
    =
    Q\mathbf{x}'_i+\mathbf{t}
    \label{eq:proof_updated_coordinate}
\end{equation}

The scalar, vector, and coordinate outputs therefore satisfy the required
$E(n)$ transformation laws.
\end{proof}

\begin{remark}[Invariant attention]
The same argument applies when the coefficients $\nu_{ij}$ are obtained from
receiver-normalized attention over invariant edge features. The attention
weights remain invariant under $E(n)$, so replacing degree-normalized
coefficients by these invariant scalars does not alter the covariance of the
aggregated vector messages.
\end{remark}

\subsection{Proof of Theorem~\ref{thm:cylindrical}:
Stabilizer Subequivariance}
\label{app:proof_thm4}

\begin{restatedtheorem}
[Stabilizer Subequivariance]
{thm:cylindrical}
\enspace
Let $\mathbf{g}\neq\mathbf{0}$ be held fixed under transformations of the
input geometry and define
\begin{equation}
    O_{\mathbf{g}}(n)
    =
    \left\{
        Q\in O(n):
        Q\mathbf{g}=\mathbf{g}
    \right\}
\end{equation}
An ESNN conditioned on Equation~\ref{eq:relaxed_embed} is equivariant under
translations and every $Q\in O_{\mathbf{g}}(n)$ with guaranteed symmetry
group
\begin{equation}
    E_{\mathbf{g}}(n)
    =
    O_{\mathbf{g}}(n)\ltimes\mathbb{R}^{n}
\end{equation}
When $\lambda_{\mathbf g}^{(\ell)}=0$ for every layer, the directional
conditioning vanishes and full $E(n)$-equivariance is recovered.
\end{restatedtheorem}

\begin{proof}
At layer $\ell$, the directional feature is
\begin{equation}
    z_{ij}^{\mathbf g,(\ell)}
    =
    \lambda_{\mathbf g}^{(\ell)}
    \left\langle
        \mathbf{r}_{ij},
        \mathbf g
    \right\rangle
    \label{eq:proof_directional_feature}
\end{equation}
Translations leave $\mathbf{r}_{ij}$ unchanged, so this feature is
translation invariant.

Now let $Q\in O_{\mathbf g}(n)$. Since
$Q\mathbf g=\mathbf g$ and $Q$ is orthogonal,
$Q^\top\mathbf g=\mathbf g$. Therefore
\begin{equation}
    \left\langle
        Q\mathbf{r}_{ij},
        \mathbf g
    \right\rangle
    =
    \left\langle
        \mathbf{r}_{ij},
        Q^\top\mathbf g
    \right\rangle
    =
    \left\langle
        \mathbf{r}_{ij},
        \mathbf g
    \right\rangle
    \label{eq:proof_stabilizer_projection}
\end{equation}
Because $\lambda_{\mathbf g}^{(\ell)}$ is a scalar model parameter,
$z_{ij}^{\mathbf g,(\ell)}$ is invariant under
$O_{\mathbf g}(n)$ at every layer.

The remaining edge features are invariant by
Equation~\ref{eq:proof_edge_context_invariant}. Hence all scalar coefficients
generated from the augmented edge context remain unchanged under
$O_{\mathbf g}(n)$, while the spatial transport operators retain the
covariance laws of Section~\ref{sec:transports}. The proof of
Theorem~\ref{thm:layer_equivariance} therefore applies with the orthogonal
action restricted from $O(n)$ to $O_{\mathbf g}(n)$. Together with
translation equivariance, this gives
\begin{equation}
    E_{\mathbf g}(n)
    =
    O_{\mathbf g}(n)\ltimes\mathbb{R}^{n}
\end{equation}

For a transformation $Q\notin O_{\mathbf g}(n)$, the directional projection
is not invariant in general because
\begin{equation}
    Q^\top\mathbf g
    \neq
    \mathbf g
\end{equation}
and hence there exist displacements for which
\begin{equation}
    \left\langle
        Q\mathbf r,
        \mathbf g
    \right\rangle
    \neq
    \left\langle
        \mathbf r,
        \mathbf g
    \right\rangle
\end{equation}
If $\lambda_{\mathbf g}^{(\ell)}\neq0$ for at least one layer, the
architecture therefore does not guarantee equivariance outside
$O_{\mathbf g}(n)$. Particular learned parameters may ignore the directional
feature and realize a larger symmetry, but that larger symmetry is not
enforced by the construction.

Finally, if
$\lambda_{\mathbf g}^{(\ell)}=0$
for every layer, all directional features vanish and the network reduces to
the fully $E(n)$-equivariant architecture of
Section~\ref{sec:layer}.
\end{proof}

\begin{remark}[Learned preferred direction]
The same symmetry statement applies when $\mathbf g$ is learned jointly with
the network provided its learned value is non-zero. For fixed model
parameters, $\mathbf g$ is a single global ambient parameter and is not
transformed together with the input. The guaranteed symmetry is therefore
the stabilizer of its learned value. If $\mathbf g=\mathbf{0}$, the
directional feature vanishes and no symmetry reduction is induced.
\end{remark}

\section{Experimental Details}
\label{app:experimental_details}

This appendix describes the datasets, geometric inputs, task-specific lifting
procedures, training objectives, model selection, and evaluation protocols used
in our experiments. Selected configurations and model-selection details are
summarized in Appendix~\ref{app:selection_multiseed}.

\subsection{Charged N-Body Dynamics}
\label{app:nbody_details}

\paragraph{Dataset and prediction task.}
We follow the standard charged N-body benchmark protocol used for EGNN
comparisons in prior work~\citep{satorras2022enequivariantgraphneural}. Each
example contains five charged particles in three dimensions. The input consists
of the initial position $\mathbf{x}_i$, velocity $\mathbf{v}_i$, and charge
$q_i\in\{-1,+1\}$ of every particle, and the target is the particle position
after $1{,}000$ simulation steps:
\begin{equation}
    \left\{\mathbf{x}^{(0)}_i,\mathbf{v}^{(0)}_i,q_i\right\}_{i=1}^{5}
    \longmapsto
    \left\{\widehat{\mathbf{x}}^{(1000)}_i\right\}_{i=1}^{5}
\end{equation}
The fixed train/validation/test split contains
$3{,}000/2{,}000/2{,}000$ simulated systems. Initial positions are Gaussian,
initial velocity directions are Gaussian and rescaled to norm $0.5$, and
charges are sampled independently and uniformly from $\{-1,+1\}$. No
observation noise is added.

The charged-particle literature baselines in
Table~\ref{tab:nbody_mse}(a) use evaluations of the same five-particle
benchmark. SE(3)-Transformer, TFN, Radial Field, and GMN follow the five-run
comparison of \citet{liu2024segno}; EGNN and GVP-GNN follow
\citet{liu2024multivector}; and FA-GNN follows
\citet{levy2023multiple}. Their uncertainties are reported as given in the
corresponding sources.

\paragraph{Graph and lifting.}
Each system is represented as a complete directed graph without self-loops.
ESNN uses the charge $q_i$ as the invariant scalar feature $\mathbf{s}_i$,
the particle position as the coordinate state, and the velocity $\mathbf{v}_i$
as a covariant vector input. The lift initializes the first vector group of
$\mathbf{V}_i$ from $\mathbf{v}_i$, replicated across its vector channels.
Although the Identity model has no latent vector group, its coordinate update
remains conditioned on velocity.

\paragraph{Training and evaluation.}
Models minimize mean squared coordinate error,
\begin{equation}
    \mathcal{L}_{\mathrm{NBody}}
    =
    \frac{1}{BN}
    \sum_{b=1}^{B}\sum_{i=1}^{N}
    \left\|
        \widehat{\mathbf{x}}^{(1000)}_{b,i}
        -\mathbf{x}^{(1000)}_{b,i}
    \right\|_2^2
\end{equation}
Validation MSE selects the checkpoint, and the corresponding held-out test MSE
is reported. All selected ESNN models use batch size $100$, AdamW, cosine
decay, twenty warm-up epochs, degree-normalized aggregation, and gradient
clipping at $5$.

\subsection{Gravity-Augmented N-Body Dynamics}
\label{app:nbody_gravity_details}
This task follows the same graph construction, inputs, data split, and
coordinate-MSE objective as the charged setting, with an additional uniform
gravitational acceleration
\begin{equation}
    \mathbf{a}_{\mathbf g}=(0,0,-9.81)^\top,
    \qquad
    \widehat{\mathbf g}_{\mathrm{true}}=(0,0,-1)^\top
    \label{eq:true_gravity_direction}
\end{equation}
The model predicts particle positions $1{,}000$ integration steps ahead, with
$\Delta t=10^{-3}$ and therefore a physical horizon of $1.0$:
\begin{equation}
    \left\{\mathbf{x}^{(0)}_i,\mathbf{v}^{(0)}_i,q_i\right\}_{i=1}^{5}
    \longmapsto
    \left\{\widehat{\mathbf{x}}^{(1000)}_i\right\}_{i=1}^{5}
    \label{eq:nbody_gravity_prediction}
\end{equation}
The simulation contains no floor or boundary conditions. For numerical
stability, the summed Coulomb force is clipped componentwise before adding
gravity. Translation symmetry remains exact, while the
$O_{\mathbf g}(3)$ symmetry of the ideal dynamics is approximate on steps
where the clipping is active. Figure~\ref{fig:nbody_gravity_comparison}
illustrates the charged and gravity-augmented settings.

\begin{figure*}[t]
    \centering
    \includegraphics[width=0.82\textwidth]{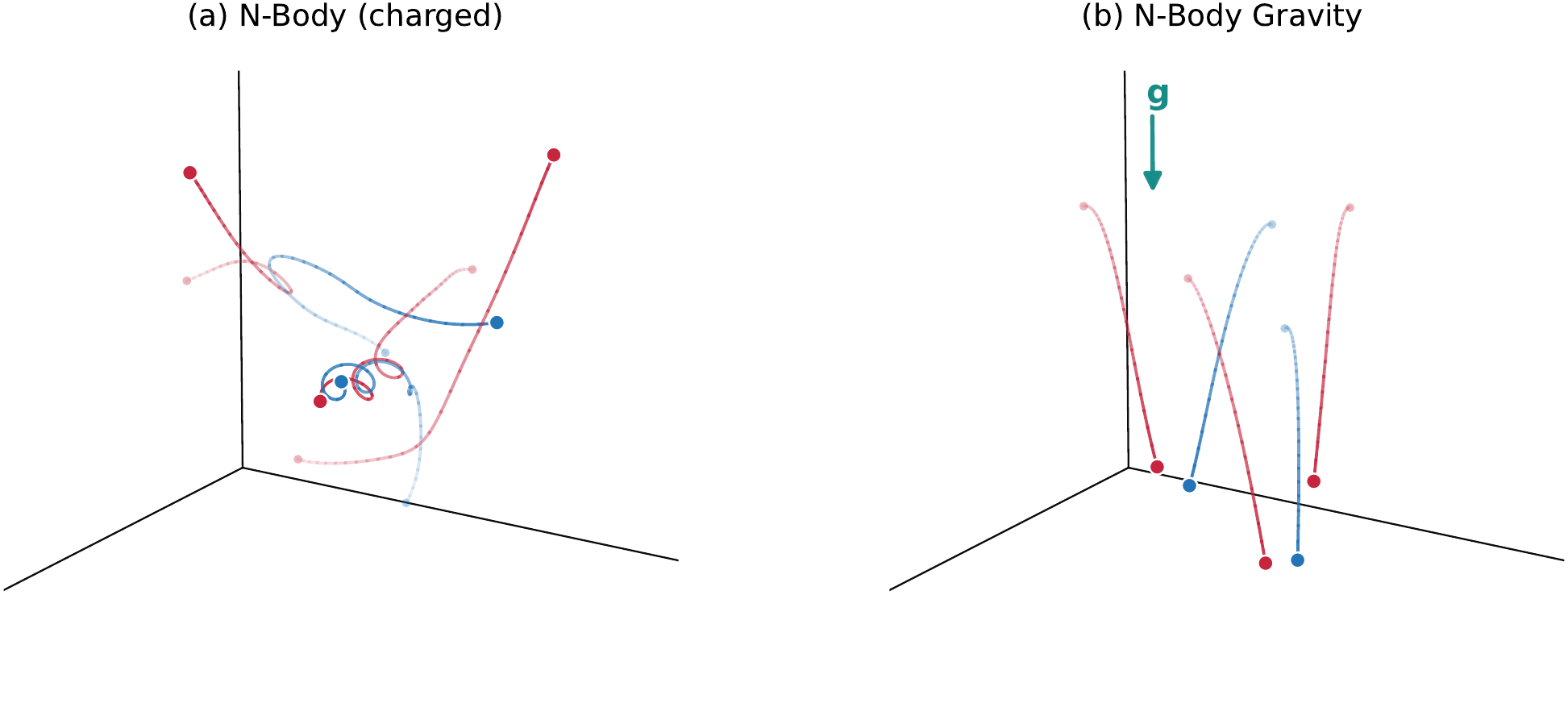}
    \caption{Charged-particle dynamics without and with a uniform gravitational
    field. Translation symmetry is preserved, while the ideal unclipped
    gravity--Coulomb dynamics retain only orthogonal transformations that
    preserve $\widehat{\mathbf g}_{\mathrm{true}}$.}
    \label{fig:nbody_gravity_comparison}
\end{figure*}

\paragraph{Preferred-direction variants.}
For each non-identity transport we compare three settings. \emph{None}
contains no preferred-direction pathway and remains fully $E(3)$-equivariant.
\emph{Fixed} supplies the normalized physical direction
$\mathbf g=\widehat{\mathbf g}_{\mathrm{true}}$ as a non-trainable vector,
while \emph{Learned} uses one global trainable vector $\mathbf g_\theta$,
shared across layers and initialized as
$\mathbf g_\theta=10^{-2}\boldsymbol{\epsilon}$ with
$\boldsymbol{\epsilon}\sim\mathcal N(\mathbf 0,I_3)$. No auxiliary alignment
loss or direct supervision of $\mathbf g_\theta$ is used. Each active layer
receives
\begin{equation}
    z_{ij}^{\mathbf g,(\ell)}
    =
    \lambda_{\mathbf g}^{(\ell)}
    \left\langle\mathbf r_{ij},\mathbf g\right\rangle
    \label{eq:gravity_projection}
\end{equation}
where the layer-specific $\lambda_{\mathbf g}^{(\ell)}$ is initialized exactly
at zero. Within each transport, the same selected backbone and training protocol are used for all three direction modes. Identity Transport is omitted because it has no directional transport pathway.

For the \emph{Learned} setting we measure recovery of the gravity axis up to
sign:
\begin{equation}
    A_{\mathbf g}
    =
    \left|
        \left\langle
            \frac{\mathbf g_\theta}{\|\mathbf g_\theta\|_2},
            \widehat{\mathbf g}_{\mathrm{true}}
        \right\rangle
    \right|
    \label{eq:gravity_alignment}
\end{equation}
For both symmetry-relaxed settings we additionally report
$\max_\ell|\lambda_{\mathbf g}^{(\ell)}|\,\|\mathbf g\|_2$,
where $\mathbf g=\widehat{\mathbf g}_{\mathrm{true}}$ for \emph{Fixed} and
$\mathbf g=\mathbf g_\theta$ for \emph{Learned}. This measures the effective
scale of the directional pathway after training.

\subsection{ModelNet40 Point-Cloud Classification}
\label{app:modelnet40_details}

Figure~\ref{fig:modelnet40_task} summarizes the point-cloud representation and
rotation protocols used in this task.

\begin{figure}[t]
    \centering
    \includegraphics[width=0.82\columnwidth]{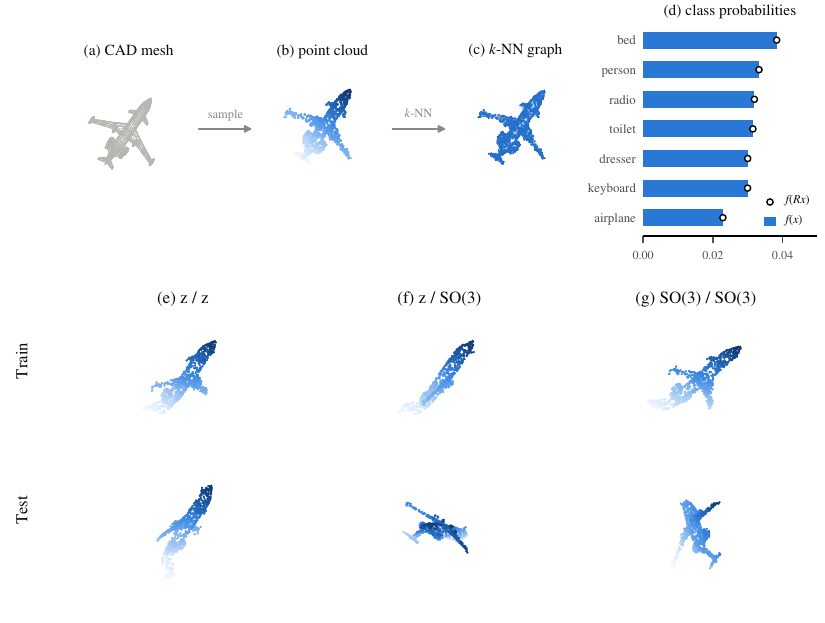}
    \caption{ModelNet40 graph construction and rotation protocols. Each
    centered point cloud is converted into a local $k$-nearest-neighbor graph
    and evaluated under the $z/z$, $z/\mathrm{SO}(3)$, and
    $\mathrm{SO}(3)/\mathrm{SO}(3)$ train/test regimes.}
    \label{fig:modelnet40_task}
\end{figure}

\paragraph{Data and graph construction.}
We use the resampled ModelNet40 dataset~\citep{wu20153d}, with $9{,}843$
training and $2{,}468$ test shapes. Each shape is represented by $1{,}024$
points, centered and rescaled to the unit ball, with surface normals discarded.
A class-stratified $10\%$ validation split is drawn from the training set using
seed $42$. We construct a symmetrized $20$-nearest-neighbor graph without
self-loops. The in-repository EGNN baseline uses the same data, split, graph construction,
rotation protocols, and evaluation procedure as ESNN. Additional baselines in
Table~\ref{tab:modelnet40_results} are taken from
\citet{lippmann2024beyond}.

\paragraph{Lifting and readout.}
Each point receives the constant invariant feature $\mathbf{s}_i=1$. The
covariant lift is
\begin{equation}
    \mathbf V_i^{\mathrm{lift}}
    =
    \left[
        \mathbf x_i,
        \frac{1}{|\mathcal N(i)|}
        \sum_{j\in\mathcal N(i)}\mathbf r_{ij}
    \right],
\end{equation}
combining the centered position with the mean incident displacement, followed
by learned channel-wise scaling. Both components transform as polar vectors
under $O(3)$. For classification, invariant vector-channel energies are pooled
together with the scalar features and passed to a two-layer classifier.

\paragraph{Rotation protocols.}
We consider two augmentation regimes. For $z$-training, training and validation
shapes are independently rotated about the vertical axis, and validation
accuracy selects an EMA checkpoint. The same checkpoint is evaluated under
vertical-axis test rotations ($z/z$) and independent Haar
$\mathrm{SO}(3)$ rotations ($z/\mathrm{SO}(3)$), so the latter directly
measures generalization to unseen three-dimensional rotations. For $\mathrm{SO}(3)/\mathrm{SO}(3)$, a separate model is trained and selected
using Haar $\mathrm{SO}(3)$ augmentation before evaluation under independent
$\mathrm{SO}(3)$ test rotations. Models are trained with cross-entropy and
label smoothing, using an EMA decay of $0.999$ for model selection and final
evaluation. Reported EGNN and ESNN results use one run per selected
configuration.

\subsection{Mesh-Based Physical Dynamics}
\label{app:meshgraphnets_details}

We evaluate \textsc{CylinderFlow}, \textsc{DeformingPlate}, and \textsc{Airfoil} from the
MeshGraphNets benchmark~\citep{pfaff2021learning}.
Figure~\ref{fig:meshgraphnets_tasks} summarizes the three tasks and their
rollout setting.

\begin{figure*}[t]
    \centering
    \includegraphics[width=0.96\textwidth]{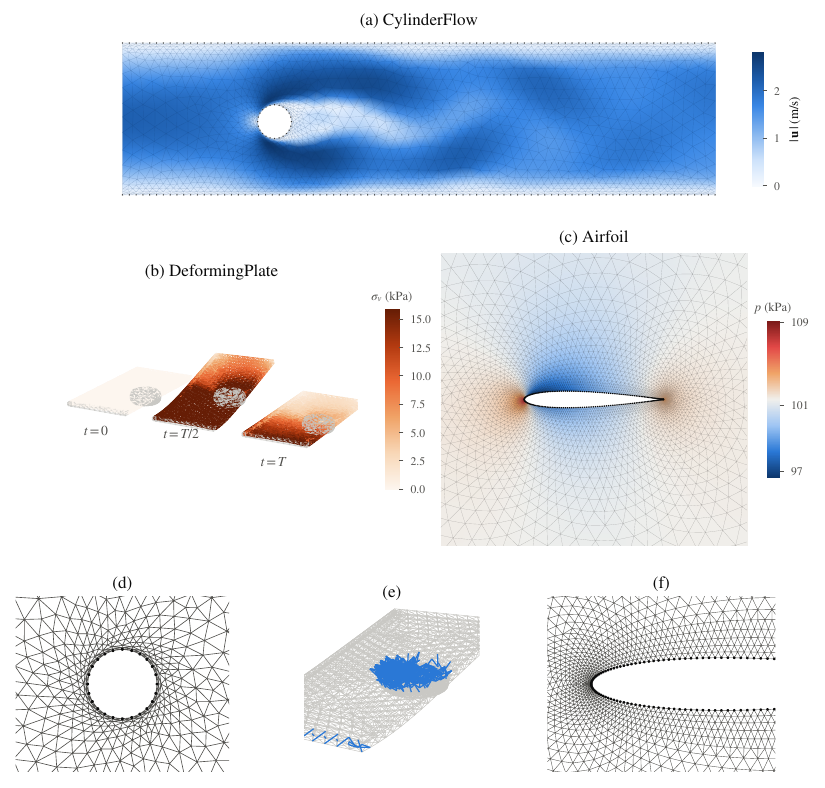}
    \caption{Mesh-based dynamics benchmarks: \textsc{CylinderFlow},
    \textsc{DeformingPlate}, and \textsc{Airfoil}. Models are trained on
    one-step targets and evaluated autoregressively.}
    \label{fig:meshgraphnets_tasks}
\end{figure*}

\paragraph{Common protocol.}
We use the official train/validation/test trajectory splits:
$1{,}000/100/100$ for \textsc{CylinderFlow} and \textsc{Airfoil}, and
$1{,}200/100/100$ for \textsc{DeformingPlate}. Mesh edges are represented in
both directions, with additional world-space edges within radius $0.03$ for
\textsc{DeformingPlate}. Scalar features are standardized componentwise and
vector features by an isotropic RMS using training-set statistics. Models are trained on randomly sampled one-step transitions with normalized
MSE and task-specific input noise. Checkpoints are selected by one-step
validation performance, while hyperparameter configurations are chosen by
validation rollout RMSE. Final models are evaluated autoregressively on the
held-out test trajectories, with prescribed boundary or kinematic nodes
clamped to the data. We report physical-space RMSE after one step, $50$ steps,
and over the full trajectory, averaging trajectories equally. MeshGraphNet, EGNN, and ESNN use the same processed datasets, splits,
normalization, prediction targets, evaluation supports, and rollout procedure,
with results reported over five runs.

\paragraph{CylinderFlow.}
The scalar input is a five-way node-type encoding, with current velocity as the
covariant vector input. The model predicts the velocity increment and next
pressure. The ESNN lift combines velocity with the mean incident edge
displacement. The primary rollout metric is velocity RMSE over normal and
outflow nodes.

\paragraph{DeformingPlate.}
The scalar input is a three-way node-type encoding, with current world position
as the coordinate state. Edge features encode edge type, rest length, and
strain. The covariant lift combines the prescribed next displacement on
kinematic nodes with the mean incident edge displacement. The model predicts
displacement and $\log(1+\text{von-Mises stress})$, with world-position RMSE
over nonkinematic nodes as the primary metric.

\paragraph{Airfoil.}
The scalar input combines a nine-way node-type encoding with current density,
while current velocity provides the covariant vector input. The model predicts
velocity and density increments together with next pressure. The ESNN lift
combines velocity with the mean incident edge displacement, and the primary
rollout metric is velocity RMSE over normal nodes. For \textsc{Airfoil}, this evaluation support differs from the near-wing region
used for the headline result in the original MeshGraphNets study. We therefore
compare ESNN directly with the in-repository MeshGraphNet and EGNN models
evaluated on the same support.

\subsection{QM9 Molecular-Property Prediction}
\label{app:qm9_details}

\paragraph{Dataset and reporting convention.}
We use the QM9 data partition and reporting convention of
GotenNet~\citep{aykent2025gotennet}. The PyTorch Geometric release contains
$130{,}831$ molecules after standard preprocessing. A fixed random
$110{,}000/10{,}000/10{,}831$ train/validation/test partition is generated
with split seed $1$. Molecular graphs are radius graphs with cutoff
$5\,\text{\AA}$, at most $32$ neighbors per atom, and no self-loops.

The twelve independently trained targets are
\begin{equation}
    \left\{
    \mu,\alpha,\epsilon_{\mathrm{HOMO}},\epsilon_{\mathrm{LUMO}},
    \Delta\epsilon,\langle R^2\rangle,\mathrm{ZPVE},U_0,U,H,G,C_v
    \right\}.
\end{equation}
Validation and test performance are reported as MAE in the same units:
mD for $\mu$, $\mathrm{m}a_0^3$ for $\alpha$, meV for orbital and energy
targets, $\mathrm{m}a_0^2$ for $\langle R^2\rangle$,
$\mathrm{mcal}\,\mathrm{mol}^{-1}\,\mathrm{K}^{-1}$ for $C_v$, and meV for
ZPVE. Literature results obtained with a different data partition are marked
explicitly in Table~\ref{tab:qm9_results}.

\paragraph{Inputs, lifting, and readout.}
PyTorch Geometric atom descriptors provide the invariant scalar features
$\mathbf{s}_i$, while atomic numbers are retained for atom-reference
corrections. Geometry is encoded through the radius graph and relative
displacements, without explicit bond attributes. The covariant lift initializes
vector features from centered atomic positions with learned scales. The readout
combines invariant vector norms with scalar features, followed by
permutation-invariant pooling and a scalar prediction head.

\paragraph{Training and model selection.}
A separate model is trained for each property using MSE on standardized targets.
For $U_0,U,H,$ and $G$, atom-reference contributions are removed before
normalization and restored at evaluation. Validation MAE selects the checkpoint
for each property and transport, which is then evaluated once on the held-out
test set. Table~\ref{tab:qm9_results} reports the resulting test MAE.

\begin{table*}[t]
\centering
\scriptsize
\setlength{\tabcolsep}{2.8pt}
\renewcommand{\arraystretch}{0.95}

\caption{
\emph{QM9 molecular-property prediction.}
Mean absolute error (MAE) on twelve invariant molecular properties.
Baseline values and partition annotations follow
\citet{aykent2025gotennet}; $\dagger$ marks a different data partition.
ESNN results are reported for the validation-selected model for each property,
with the best ESNN result shown in \textbf{bold}. Units are given in the
second header row. Lower is better.
}
\label{tab:qm9_results}

\resizebox{\textwidth}{!}{%
\begin{tabular}{@{}lcccccccccccc@{}}
\toprule

\textbf{Method}
& $\boldsymbol{\alpha}$
& $\boldsymbol{\Delta\epsilon}$
& $\boldsymbol{\epsilon_{\mathrm{HOMO}}}$
& $\boldsymbol{\epsilon_{\mathrm{LUMO}}}$
& $\boldsymbol{\mu}$
& $\boldsymbol{C_v}$
& $\boldsymbol{G}$
& $\boldsymbol{H}$
& $\boldsymbol{\langle R^2\rangle}$
& $\boldsymbol{U}$
& $\boldsymbol{U_0}$
& \textbf{ZPVE} \\

\textbf{Units}
& $\mathrm{m}a_0^3$
& meV
& meV
& meV
& mD
& $\mathrm{mcal\,mol^{-1}\,K^{-1}}$
& meV
& meV
& $\mathrm{m}a_0^2$
& meV
& meV
& meV \\

\midrule
\multicolumn{13}{l}{\emph{Invariant models}} \\
\midrule

Cormorant
& 85
& 61
& 34
& 38
& 38
& 26
& 20
& 21
& 961
& 21
& 22
& 2.03 \\

NMP
& 92
& 69
& 43
& 38
& 30
& 40
& 19
& 17
& 180
& 20
& 20
& 1.50 \\

DimeNet++$\dagger$
& 44
& 32.6
& 24.6
& 19.5
& 29.7
& 23
& 7.56
& 6.53
& 331
& 6.28
& 6.32
& 1.21 \\

ComENet$\dagger$
& 45
& 32.4
& 23.1
& 19.8
& 24.5
& 22
& 7.98
& 6.86
& 259
& 6.82
& 6.69
& 1.20 \\

SphereNet$\dagger$
& 46
& 31.1
& 22.8
& 18.9
& 24.5
& 22
& 7.78
& 6.33
& 268
& 6.36
& 6.26
& 1.12 \\

\midrule
\multicolumn{13}{l}{\emph{Scalarization-based models}} \\
\midrule

ClofNet
& 63
& 53
& 33
& 25
& 40
& 27
& 9
& 9
& 610
& 9
& 8
& 1.23 \\

EGNN
& 71
& 48
& 29
& 25
& 29
& 31
& 12
& 12
& 106
& 12
& 11
& 1.55 \\

PaiNN$\dagger$
& 45
& 45.7
& 27.6
& 20.4
& 12.0
& 24
& 7.35
& 5.98
& 66
& 5.83
& 5.85
& 1.28 \\

LEFTNet
& 48
& 40
& 24
& 18
& 12
& 23
& 7
& 6
& 109
& 7
& 6
& 1.33 \\

EQGAT
& 53
& 32
& 20
& 16
& 11
& 24
& 23
& 24
& 382
& 25
& 25
& 2.00 \\

ET
& 59
& 36.1
& 20.3
& 17.5
& 11
& 26
& 7.62
& 6.16
& 33
& 6.38
& 6.15
& 1.84 \\

Geoformer
& 40
& 33.8
& 18.4
& 15.4
& 10
& 22
& 6.13
& 4.39
& 28
& 4.41
& 4.43
& 1.28 \\

SaVeNet-B$\dagger$
& 39
& 24.8
& 18.4
& 16.3
& 9.3
& 23
& 6.64
& 5.43
& 58
& 5.48
& 5.43
& 1.18 \\

\midrule
\multicolumn{13}{l}{\emph{Equivariant Sheaf Neural Networks}} \\
\midrule

\textbf{ESNN-Id}
& 63.08
& 65.35
& 43.13
& 28.52
& 58.81
& 37.83
& 15.76
& 18.71
& 749.44
& 10.67
& 10.54
& 1.54 \\

\textbf{ESNN-Diag}
& 62.16
& 43.69
& \textbf{20.29}
& 21.43
& 17.15
& \textbf{30.57}
& \textbf{11.48}
& \textbf{12.09}
& 110.88
& 11.77
& 13.57
& 1.62 \\

\textbf{ESNN-Ortho}
& 79.42
& 43.35
& 25.10
& 23.17
& 21.77
& 36.50
& 17.76
& 26.07
& 323.52
& 18.59
& 25.25
& 1.94 \\

\textbf{ESNN-RadTan}
& \textbf{55.90}
& \textbf{39.01}
& 25.16
& \textbf{19.17}
& \textbf{13.99}
& 31.57
& 14.77
& 13.86
& \textbf{108.16}
& \textbf{10.63}
& \textbf{10.34}
& \textbf{1.45} \\

\bottomrule
\end{tabular}%
}
\end{table*}

Across the ESNN variants, Radial--Tangential Transport gives the lowest MAE
on eight of the twelve properties, while Diagonal Transport is best on the
remaining four. Relative to EGNN, the best ESNN result is lower on ten of the
twelve targets, with $H$ and $\langle R^2\rangle$ as the exceptions.

\subsection{Implementation, Model Selection, and Evaluation}
\label{app:selection_multiseed}

\paragraph{Training and model selection.}
Unless stated otherwise, ESNN and EGNN are trained with AdamW, cosine
learning-rate decay, degree-normalized aggregation, and gradient clipping;
MeshGraphNet uses Adam. Dataset splits and preprocessing statistics are fixed
before training.

Hyperparameters are selected using validation data only. Charged and gravity
N-body models are selected by validation MSE, ModelNet40 by validation
accuracy, mesh models by validation rollout RMSE, and QM9 models by validation
MAE. For the mesh tasks, checkpoints are selected using one-step validation
performance before configurations are compared by rollout RMSE.

For ModelNet40, $z/\mathrm{SO}(3)$ is not a separate training or selection
regime. The same validation-selected $z$-trained checkpoint is evaluated under
both $z$ and $\mathrm{SO}(3)$ test rotations.

\begin{table*}[t]
\centering
\caption{Selected N-body, gravity, and ModelNet40 configurations.
$C$, $L$, and $V$ denote scalar width, message-passing layers, and vector
groups. For ModelNet40, $z/z$ and $z/\mathrm{SO}(3)$ use the same
$z$-trained checkpoint.}
\label{tab:selected_core}
\small
\setlength{\tabcolsep}{5pt}
\begin{tabular}{llllrrrr}
\toprule
Task & Model & Transport / training regime & $C$ & $L$ & $V$ & Learning rate & Epochs \\
\midrule
Charged N-body
& ESNN & Identity
& 64 & 6 & 0 & $3.38\!\times\!10^{-3}$ & 10000 \\

& ESNN & Diagonal
& 32 & 5 & 1 & $3.25\!\times\!10^{-3}$ & 3000 \\

& ESNN & Orthogonal
& 32 & 7 & 1 & $1.92\!\times\!10^{-3}$ & 3000 \\

& ESNN & Radial--Tangential
& 32 & 5 & 1 & $3.41\!\times\!10^{-3}$ & 3000 \\

\midrule
Gravity N-body
& ESNN & Diagonal
& 64 & 7 & 1 & $1.01\!\times\!10^{-3}$ & 3000 \\

& ESNN & Orthogonal
& 32 & 7 & 1 & $1.24\!\times\!10^{-3}$ & 3000 \\

& ESNN & Radial--Tangential
& 32 & 5 & 1 & $3.53\!\times\!10^{-3}$ & 3000 \\

\midrule
ModelNet40
& EGNN & $z$-trained
& 304 & 4 & -- & $5\!\times\!10^{-4}$ & 800 \\

& EGNN & $\mathrm{SO}(3)$-trained
& 224 & 4 & -- & $10^{-3}$ & 800 \\

& ESNN & Identity, $z$-trained
& 128 & 4 & 2 & $10^{-3}$ & 800 \\

& ESNN & Identity, $\mathrm{SO}(3)$-trained
& 128 & 4 & 2 & $10^{-3}$ & 800 \\

& ESNN & Diagonal, $z$-trained
& 64 & 3 & 2 & $1.06\!\times\!10^{-3}$ & 300 \\

& ESNN & Diagonal, $\mathrm{SO}(3)$-trained
& 128 & 4 & 2 & $10^{-3}$ & 800 \\

& ESNN & Orthogonal, $z$-trained
& 128 & 4 & 2 & $10^{-3}$ & 800 \\

& ESNN & Orthogonal, $\mathrm{SO}(3)$-trained
& 128 & 4 & 2 & $10^{-3}$ & 800 \\

& ESNN & Radial--Tangential, $z$-trained
& 128 & 4 & 2 & $10^{-3}$ & 800 \\

& ESNN & Radial--Tangential, $\mathrm{SO}(3)$-trained
& 128 & 4 & 2 & $10^{-3}$ & 800 \\
\bottomrule
\end{tabular}
\end{table*}

Most ModelNet40 ESNN experiments use $C=128$, $L=4$, $V=2$, and learning
rate $10^{-3}$. The selected $z$-trained Diagonal model uses $C=64$, $L=3$,
$V=2$, and learning rate $1.06\times10^{-3}$.

\begin{table*}[t]
\centering
\caption{Selected mesh configurations. All models use the same processed
datasets and rollout evaluation. For MeshGraphNet, width denotes latent width
and depth the number of message-passing steps; for ESNN, $V$ denotes vector
groups.}
\label{tab:selected_mesh}
\small
\setlength{\tabcolsep}{4pt}
\begin{adjustbox}{max width=\textwidth}
\begin{tabular}{llllrrrrr}
\toprule
Task & Model & Transport & Width & Depth & $V$ & Batch & Learning rate & Epochs \\
\midrule
Airfoil
& MeshGraphNet & -- & 192 & 10 & -- & 2 & $2\!\times\!10^{-4}$ & 45 \\
& EGNN & -- & 246 & 6 & -- & 2 & $2.5\!\times\!10^{-4}$ & 45 \\
& ESNN & Identity & 128 & 7 & 2 & 2 & $3.01\!\times\!10^{-4}$ & 45 \\
& ESNN & Diagonal & 128 & 5 & 3 & 2 & $2.48\!\times\!10^{-4}$ & 45 \\
& ESNN & Orthogonal & 128 & 7 & 2 & 2 & $5.86\!\times\!10^{-4}$ & 45 \\
& ESNN & Radial--Tangential & 64 & 8 & 3 & 2 & $3.93\!\times\!10^{-4}$ & 45 \\

\midrule
CylinderFlow
& MeshGraphNet & -- & 192 & 15 & -- & 8 & $5\!\times\!10^{-5}$ & 52 \\
& EGNN & -- & 244 & 8 & -- & 8 & $2.5\!\times\!10^{-4}$ & 52 \\
& ESNN & Identity & 128 & 7 & 2 & 8 & $3.91\!\times\!10^{-4}$ & 52 \\
& ESNN & Diagonal & 128 & 9 & 1 & 8 & $1.63\!\times\!10^{-4}$ & 52 \\
& ESNN & Orthogonal & 128 & 8 & 2 & 8 & $6.29\!\times\!10^{-4}$ & 52 \\
& ESNN & Radial--Tangential & 64 & 9 & 2 & 8 & $3.33\!\times\!10^{-4}$ & 52 \\

\midrule
DeformingPlate
& MeshGraphNet & -- & 192 & 10 & -- & 4 & $5\!\times\!10^{-5}$ & 16 \\
& EGNN & -- & 318 & 8 & -- & 4 & $3\!\times\!10^{-4}$ & 16 \\
& ESNN & Identity & 128 & 8 & 2 & 4 & $3\!\times\!10^{-4}$ & 16 \\
& ESNN & Diagonal & 128 & 10 & 1 & 4 & $3.69\!\times\!10^{-4}$ & 16 \\
& ESNN & Orthogonal & 128 & 8 & 3 & 2 & $2.98\!\times\!10^{-4}$ & 16 \\
& ESNN & Radial--Tangential & 128 & 8 & 2 & 4 & $3\!\times\!10^{-4}$ & 16 \\
\bottomrule
\end{tabular}
\end{adjustbox}
\end{table*}

\paragraph{Repeated-run evaluation.}
Charged N-body, gravity N-body, and the mesh benchmarks are evaluated over five
runs using fixed train/validation/test partitions.
Reported uncertainties are sample standard deviations across these runs.

ModelNet40 is reported from one run per selected configuration. For the
$z$-trained models, the $z/z$ and $z/\mathrm{SO}(3)$ evaluations use identical
trained weights. QM9 uses a separately validation-selected model for each
property and transport, evaluated once on the held-out test set. Literature
baselines follow the reporting conventions of their original sources.

\newpage

\end{document}